%% file: NewSPA-arXiv.tex
\documentclass{article} 
\usepackage{iclr2024_conference,times}
\usepackage{adjustbox}
\usepackage{enumitem}
\usepackage{hyperref}

\usepackage{amsthm}

\usepackage{algorithm}
\usepackage{algpseudocode,wrapfig}

\input{math_commands.tex}

\usepackage{hyperref}
\usepackage{url} 
\newcommand{\calH}{{\cal H}} 
\newcommand{\calS}{{\cal S}}
\newcommand{\goto}{\rightarrow}

\title{Improved algorithm and bounds for successive projection}

\author{Jiashun Jin \& Gabriel Moryoussef\\
Department of Statistics\\
Carnegie Mellon University\\
Pittsburgh, PA 15213, USA \\
\texttt{\{jiashun, gmoryous\}@andrew.cmu.edu} \\
\AND
Zheng Tracy Ke \& Jiajun Tang \& Jingming Wang \\
Department of Statistics \\
Harvard University \\
Cambridge, MA 02138, USA \\
\texttt{\{zke,jiajuntang,jingmingwang\}@fas.harvard.edu} 
}

\iclrfinalcopy 
\begin{document}

\maketitle


\begin{abstract} 
Given a $K$-vertex simplex in a $d$-dimensional  space, suppose 
we measure $n$ points on the 
simplex with noise (hence, some of the observed points fall outside the simplex). 
Vertex hunting is the problem of estimating the $K$ vertices of the simplex. A popular vertex hunting algorithm is   
successive projection algorithm (SPA). However,  SPA is observed to perform unsatisfactorily under strong noise or outliers.  We propose pseudo-point SPA  (pp-SPA). It uses a projection step and a denoise step to generate pseudo-points and feed them into SPA for vertex hunting. 
We derive error bounds for pp-SPA, leveraging on extreme value theory of (possibly) high-dimensional random vectors. The results suggest that pp-SPA has faster rates and better numerical performances than SPA. 
Our analysis includes an improved non-asymptotic bound for the original SPA, which is of independent interest. 
\end{abstract}

\tableofcontents

\section{Introduction}
Fix $d \geq 1$ and suppose we observe $n$ vectors $X_1, X_2, \ldots, X_n$ in $\mathbb{R}^d$, where 
\begin{equation} \label{model1} 
X_i = r_i  + \eps_i,  \qquad \eps_i \stackrel{iid}{\sim} N(0,  
\sigma^2 I_d). 
\end{equation} 
The Gaussian assumption is for technical simplicity and can be relaxed. 
For an integer $1 \leq K \leq d+1$,  we assume that there is a  
simplex with $K$ vertices ${\cal S}_0$ on the hyperplane ${\cal H}_0$ such that 
each $r_i$ falls within the simplex (note that a simplex with $K$ vertices always falls on a 
 $(K-1)$-dimensional hyperplane of $\mathbb{R}^d$).  In other words, let 
$v_1, v_2, \ldots, v_K\in\mathbb{R}^d$  be the vertices of the simplex and let $V = [v_1, v_2, \ldots, v_K]$.   
We assume that for each $1 \leq i \leq n$, there is a $K$-dimensional weight vector $\pi_i$ (a weight vector is vector where all entries are non-negative with a  unit sum) such that 
\begin{equation} \label{model2}
r_i = \sum_{k = 1}^K \pi_i(k) v_k =   V \pi_i. 
\end{equation} 
Here, $\pi_i$'s are unknown but are of major interest, and to estimate $\pi_i$, the key is vertex hunting (i.e., estimating the $K$ vertices of the simplex ${\cal S}_0$).  
In fact, once the vertices are estimated, we can estimate $\pi_1, \pi_2, \ldots, \pi_n$ by the relationship of $X_i \approx r_i = V \pi_i$.  Motivated by these, the primary interest of this paper is  vertex hunting (VH). The problem may arise in many application areas. 
{\it (1) Hyper-spectral unmixing}: 
Hyperspectral unmixing \citep{bioucas2012hyperspectral} is the problem of separating the pixel spectra from a hyperspectral image into a collection of constituent spectra. 
$X_i$ contains the spectral measurements of pixel $i$ at $d$ different channels, $v_1,\ldots,v_K$ are the constituent spectra (called {\it endmembers}), and $\pi_i$ contains the fractional {\it abundances} of endmembers at pixel $i$. It is of great interest to identify the endmembers and estimate the abundances. 
{\it (2) Archetypal analysis}.  Archytypal analysis \citep{cutler1994archetypal} is a useful tool for representation learning.  Take its application in genetics for example \citep{satija2015spatial}. Each $X_i$ is the gene expression of cell $i$, and each $v_k$ is an archetypal expression pattern. Identifying these archetypal expression patterns is useful for inferring a transcriptome-wide map of spatial patterning. 
{\it (3) Network membership estimation}.   Let $A \in \mathbb{R}^{n, n}$ be the adjacency matrix of an undirected network with $n$ nodes and $K$ communities.  Let $(\hat{\lambda}_k, \hat{\xi}_k)$ be the $k$-th eigenpair of $A$, and write $\widehat{\Xi} = [\hat{\xi}_1, \hat{\xi}_2, \ldots,  \hat{\xi}_K]$. 
Under certain network models (e.g., \cite{huang2023pcabm, airoldi2008mixed, JiZhuMM, SCORE-Review, GRDPG}), 
there is a $K$-vertex simplex in $\mathbb{R}^K$ such that 
for each $1 \leq i \leq n$, the $i$-th row of $\widehat{\Xi}$ falls (up to noise corruption) 
inside the simplex, and vertex hunting is an important step in community analysis. 
{\it (4) Topic modeling}. Let $D \in \mathbb{R}^{n, p}$ be the frequency of word counts of $n$ text documents,  where $p$ is the dictionary size. If $D$ follows the Hoffman's model  with $K$ topics, then there is also 
simplex in the spectral domain \citep{ke2017new}), so vertex hunting is useful.  

Existing vertex hunting approaches can be roughly divided into two lines:  constrained optimizations and stepwise algorithms. In the first line, one 
proposes an objective function and estimates the vertices by solving an optimization problem.  The minimum volume transform (MVT) \citep{craig1994minimum},  archetypal analysis (AA) \citep{cutler1994archetypal,javadi2020nonnegative},  and N-FINDER \citep{winter1999n} are approaches of this line. In the second line, one uses a stepwise algorithm which iteratively identifies one vertex of the simplex at a time.  This includes the 
 popular successive projection algorithm (SPA) \citep{SPA}.   SPA is a stepwise greedy algorithm.  It does not require an objective function (how to select the objective function may be a bit subjective),  is computationally efficient, and has a theoretical guarantee.  This makes SPA especially interesting.

{\bf Our contributions}. 
Our primary interest is to improve SPA.   Despite many good properties aforementioned, SPA is a greedy algorithm, which is vulnerable to noise and outliers, and may be significantly inaccurate.  Below, we list two reasons why SPA may underperform.  
First, typically in the literature (e.g., \cite{SPA}),  one apply the SPA directly to the $d$-dimensional data points $X_1, X_2, \ldots, X_n$, regardless of what $(K, d)$ are. However, since the true vertices $v_1,\ldots,v_K$ lie on a $(K-1)$-dimensional hyperplane, if we directly apply SPA to $X_1, X_2, \ldots, X_n$, the resultant hyperplane formed by the estimated simplex vertices is likely to deviate from the true hyperplane, due to noise corruption. This will cause inefficiency of SPA. Second, since the SPA is a greedy algorithm, it tends to be biased outward bound. When we apply SPA,  it is frequently found that most of the estimated vertices fall outside of true simplex (and some of them are faraway from the true simplex).    

\begin{wrapfigure}{r}{.45\linewidth}
        \includegraphics[width=.95\linewidth, trim=0 20 0 20]{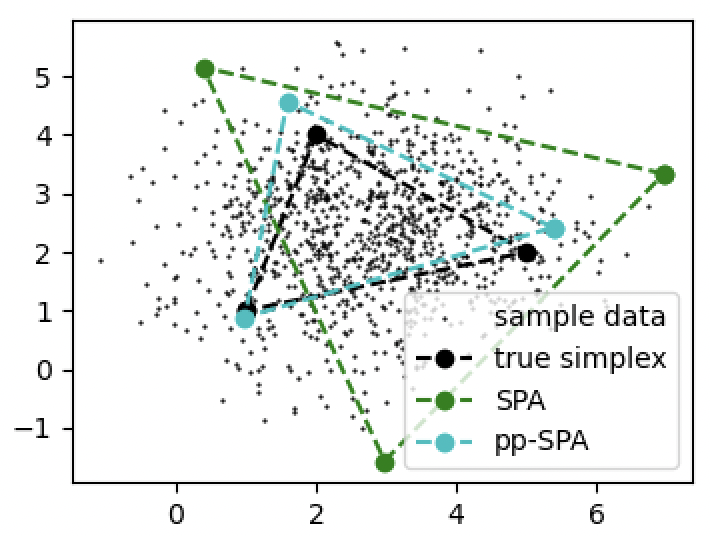}
        \caption{A numerical example ($d$=$2$, $K$=$3$).} \label{fig:triangle} 
    \end{wrapfigure}
For illustration, Figure \ref{fig:triangle} presents an example,  where $X_1, X_2, \ldots, X_n$ are generated from Model (\ref{model1}) with $(n, K, d, \sigma) = (1000, 3, 2, 1)$, and $r_i$  are uniform samples over $T$ ($T$ is the triangle with vertices $(1,1)$, $(2,4)$, and $(5,2)$). 
In this example, the true vertices (large black points) form a triangle (dashed black lines) on a $2$-dimensional hyperplane. The green and cyan-colored triangles are estimated by SPA and 
pp-SPA (our main algorithm to be introduced; since $d$ is equal to $K-1$, the hyperplane projection is skipped), 
respectively. In this example, the estimated simplex by SPA is significantly biased outward bound, suggesting a large room for improvement. Such outward bound bias of SPA is related to the design of the algorithm and is frequently observed \citep{gillis2019successive}.

To fix the issues,  we propose pseudo-point SPA (pp-SPA) 
as a new approach to vertex hunting. It contains two novel ideas as follows. First, since the simplex ${\cal S}_0$ is on 
the hyperplane $\calH_0$, we first use all data $X_1, \ldots, X_n$ to estimate the hyperplane, and then project all these points  to the hyperplane. Second, since SPA is vulnerable to noise and outliers, 
a reasonable idea is to add a denoise step before we apply SPA. 
We propose a {\it pseudo-point (pp) approach}  for denoising, 
where for each data point, we replace it by a pseudo point, computed as  the average of {\it 
all of its neighbors within a radius of $\Delta$}.  Utilizing information in the 
nearest neighborhood  is a known idea in classification \citep{HTF},  and the well-known $k$-nearest neighborhood (KNN) algorithm is such an approach. 
However,  KNN or similar ideas were never used as a denoise step for vertex hunting. 
Compared with KNN, the idea of pseudo-point approach is motivated by the underlying geometry and is for a different 
purpose. For these reasons, the idea is new at least to some extent.


We have two theoretical contributions.  
First, \cite{gillis2013fast} derived a non-asymptotic error bound for SPA,  but the bound is not always tight.  
Using a very different proof, we derive a sharper non-asymptotic bound 
for SPA. The improvement is substantial 
in the following case.   Recall that $V = [v_1, v_2, \ldots, v_K]$ and let $s_k(V)$ be the $k$-th largest singular value of $V$. The bound in \cite{gillis2013fast} is proportional to $1/s_K^2(V)$, while our 
our bound is proportional to $1/s_{K-1}^2(V)$. 
Since all vertices lie on a $(K-1)$-dimensional hyperplane, 
$s_{K-1}(V)$ is bounded away from $0$, as long as the volume of true simplex is lower bounded. However,  
 $s_K(V)$ may be  $0$ or nearly $0$; in this case,
the bound in \cite{gillis2013fast} is too conservative, but our bound 
is still valid. 
%
%
%
Second, we use our new non-asymptotic bound to derive the rate for pp-SPA, 
and show that the rate is much faster than the rate of SPA, especially when $d \gg K$. 
Even when $d=O(K)$, the bound we get for pp-SPA 
is still sharper than the bound of the original SPA. The main reason is that,  
for those points far away outside the true simplex, the corresponding 
pseudo-points we generate  are much closer to the true simplex. 
This greatly reduces the {\it outward bound biases} of SPA (see Figure~\ref{fig:triangle}). 



{\bf Related literature}. 
It was observed that SPA is susceptible to outliers, motivating several variants of SPA \citep{gillis2015semidefinite, mizutani2018efficient, gillis2019successive}. For example, \cite{bhattacharyya2020finding,bakshi2021learning,nadisic2023smoothed} modified SPA by incorporating smoothing at each iteration. In contrast, our approach involves generating all pseudo points through neighborhood averaging before executing all successive projection steps. Additionally, we exploit the fact that the simplex resides in a low-dimensional hyperplane and apply a hyperplane projection step prior to the denoising and successive projection steps. Our theoretical results surpass those existing works for several reasons: (a) we propose a new variant of SPA; (b) our analyses build upon a better version of the non-asymptotic bound than the commonly-used one in \cite{gillis2013fast}; and (c) we incorporate delicate random matrix and extreme value theory in our analysis. 


\section{A new vertex hunting algorithm} \label{sec:method} 
The successive projection algorithm (SPA) \citep{SPA} is a popular vertex hunting method. This is an iterative algorithm that estimates one vertex at a time. At each iteration, it first projects all points to the orthogonal complement of those previously found vertices and then takes the point with the largest Euclidean norm as the next estimated vertex. See Algorithm~\ref{alg:SPA} for a detailed description.  
\begin{algorithm}[htb!]
\caption{The (orthodox) Successive Projection Algorithm (SPA)} \label{alg:SPA}
{\bf Input}: $X_1, X_2, \ldots, X_n$, and $K$.

Initialize $u={\bf 0}_p$ and $y_i=X_i$, for $1\leq i\leq n$.  For $k=1,2,\ldots,K$, 
\begin{itemize}\itemsep 0pt
\item Update $y_i$ to $(I_d-uu')y_i$. Obtain $i_k=\argmax_{1\leq i\leq n}\|y_i\|$. Update $u =\|y_{i_k}\|^{-1} y_{i_k}$.  
\end{itemize}
{\bf Output}: $\hat{v}_k=X_{i_k}$, for $1\leq k\leq K$. 
\end{algorithm}

We propose pp-SPA as an improved version of the (orthodox) SPA, containing two main ideas: a {\it hyperplane projection} step and a {\it pseudo-point denoise} step. 
We now discuss  the two steps separately.

Consider the {\it hyperplane projection} step first. In our model (\ref{model2}), the noiseless points $r_1,\ldots,r_n$ live in a $(K-1)$-dimensional hyperplane. However, with noise corruption, the observed data $X_1,\ldots,X_n$ are not exactly contained in a hyperplane. Our proposal is to first use data to find a `best-fit' hyperplane and then project all data points to this hyperplane. 
Fix $d\geq K\geq 2$. Given a point $x_0\in\mathbb{R}^d$ and a projection matrix $H\in\mathbb{R}^{d\times d}$ with rank $K-1$, the $(K-1)$-dimensional hyperplane associated with $(x_0, H)$ is $\calH = \{x\in\mathbb{R}^d: (I_d - H)(x - x_0) = 0\}$. 
For any $x \in \mathbb{R}^d$,   
the Euclidean distance between $x$ and the hyperplane is equal to $\|(I_d - H)(x - x_0)\|$. Given 
$X_1, X_2, \ldots, X_n$, we aim to find a hyperplane to minimize the sum of square distances:  
\begin{equation} \label{hplane1} 
\min_{(x_0,H)}\{S(x_0, H)\}, \quad\mbox{where}\quad S(x_0, H) = \sum_{i = 1}^n \|(I_d - H) (X_i - x_0)\|^2. 
\end{equation} 
Let $Z = [Z_1,\ldots,Z_n]$, where $Z_i=X_i-\bar{X}$ and $\bar{X} = \frac{1}{n} \sum_{i =1}^n X_i$. 
For each $k$, let $u_k\in\mathbb{R}^d$ be the $k$th left singular vector of $Z$. Write $U=[u_1,\ldots,u_{K-1}]$. 
The next lemma is proved in the appendix. 
\begin{lemma}  \label{lemma:projection} 
$S(x_0,H)$ is minimized by  $x_0 = \bar{X}$ and $H = UU'$. 
\end{lemma}

For each $1 \leq i \leq n$, we first project each $X_i$ to $\tilde{X}_i$ and then transform $\tilde{X}_i$ to $Y_i$, where 
\begin{equation} \label{DefinetildeX}
\tilde{X}_i: = \bar{X} +  H(X_i - \bar{X}), \qquad  Y_i:  = U' \tilde{X}_i; \qquad  \mbox{note that $H=UU'$ and $Y_i \in \mathbb{R}^{K-1}$}. 
\end{equation} 
These steps reduce noise. To see this, we note that the true simplex lives in a hyperplane with a projection matrix $H_0=U_0U_0'$. It can be shown that $U\approx U_0$ (up to a rotation) and $Y_i \approx r_i^*  + U_0'  \eps_i$, with $r_i^*=U_0'\bar{X}+U_0' r_i$. These points $r_i^*$ still live in a simplex (in dimension $(K-1)$). 
Comparing this with the original model $X_i = r_i + \eps_i$, we see that 
$U'_0 \eps_i$ are iid samples from $N(0, \sigma^2 I_{K-1})$, 
and $\eps_i$ are iid samples from $N(0, \sigma^2 I_d)$. 
Since $K-1\ll d$ in may applications, the projection 
may significantly reduce the dimension of the noise variable. 
Later in Section~\ref{subsec:ppSPA}, we see that this implies a 
significant improvement in the convergence rate.

Next, consider the {\it neighborhood denoise} step. Fix an $\Delta > 0$ and an integer $N \geq 1$. Define the $\Delta$-neighborhood of $Y_i$ by $
B_{\Delta}(Y_i) = \{ x\in\mathbb{R}^{K-1}: \|x-Y_i\|\leq \Delta\}$. 
When there fewer than $N$ points in $B_{\Delta}(Y_i)$ (including $Y_i$ itself), 
remove $Y_i$ for the vertex hunting step next. 
Otherwise, replace $Y_i$ by the average of all points in $B_{\Delta}(Y_i)$ (denoted by 
$Y_i^*$). The main effect of the denoise effect is on the points that are {\it  far outside the simplex}. 
For these points, we either delete them for the vertex hunting step (see below),  
or replace it by a point closer to the simplex. This way, 
we pull all these points ``towards" the simplex, and thus reduce the 
estimation error in the subsequent vertex hunting step.  

Finally, we apply the (orthodox) successive projection algorithm (SPA) to $Y_1^*, Y_2^*,\cdots, Y_n^*$ and let  
$\hat{v}_1, \hat{v}_2, \ldots, \hat{v}_K$ be the estimated vertices. Let $\hat{V} = [\hat{v}_1, \hat{v}_2, 
\ldots, \hat{v}_K]$.  See Algorithm \ref{alg:PPSPA}.  
 
\begin{algorithm}[htb!]
\caption{Pseudo-Point Successive Projection Algorithm (pp-SPA)} \label{alg:PPSPA}
{\bf Input}: $X_1, X_2, \ldots, X_n \in \mathbb{R}^d$, the number of vertices $K$, and tuning parameters $(N, \Delta)$.

\begin{description}
\item[Step 1] {\it (Projection)}. Obtain $\bar{X} = \frac{1}{n}\sum_{i = 1}^n X_i$ and $Z = X-\bar{X}{\bf 1}_n'$. Let $U = [u_1,  \ldots, u_{K-1}]$ contain the first $(K-1)$ singular vectors of $Z$. For $1 \leq i \leq n$, let $Y_i = U' X_i\in\mathbb{R}^{K-1}$.
\item[Step 2] {\it (Denoise)}. Let $B_{\Delta}(Y_i)= \{ x\in\mathbb{R}^{K-1}: \|x-Y_i\|\leq \Delta\}$ denote the $\Delta$-neighborhood of $Y_i$. 
\begin{itemize}
\item If there are fewer than $N$ points (including $Y_i$ itself) in $B_{\Delta}(Y_i)$, delete this point.  
\item Otherwise, replace $Y_i$ by $Y_i^*$, which is the average of all points in $B_{\Delta}(Y_i)$. 
\end{itemize}
\item[Step 3] {\it (VH)}. Let ${\cal J}\subset\{1,\ldots,n\}$ be the set of retained points in Step 2. Apply Algorithm~\ref{alg:SPA} to $\{Y_i^*\}_{i\in {\cal J}}$ to get $\hat{v}_1^*,\hat{v}_2^*,\ldots,\hat{v}_K^*\in\mathbb{R}^{K-1}$. Let $\hat{v}_k= (I_d - H) \bar{X} + U\hat{v}_k^*\in\mathbb{R}^d$, $1\leq k\leq K$. 
\end{description}
{\bf Output}: The estimated vertices $\hat{v}_1,\ldots,\hat{v}_K$. 
\end{algorithm}

{\bf Remark 1}: The complexity of the orthodox SPA is $O(ndK)$. Regarding the complexity of pp-SPA,  it applies SPA on $(K-1)$-dimensional pseudo-points, so the complexity is $O(nK^2)$. To obtain these pseudo points, we need a projection step and a denoise step. The projection step extracts the first $(K-1)$ singular vectors of a matrix $Z (n\times d)$. Performing the whole SVD decomposition would result in $O(\min(n^2d, nd^2))$ time complexity. However, faster approach exists such as the truncated SVD which would decrease this complexity to $O(ndK)$. In the denoise step, we need to find the $\Delta$-neighborhoods for all $n$ points $Y_1,Y_2,\ldots,Y_n$. This can be made computationally efficient using the KD-Tree. The construction of KD-Tree takes $O(n\log n)$, and the search of neighbors typically takes $O\bigl(n^{(2-\frac{1}{K-1})} + nm\bigr)$, where $m$ is the maximum number of points in a neighborhood. 
 
{\bf Remark 2}: Algorithm~\ref{alg:PPSPA} has tuning parameters $(N, \Delta)$, where $\Delta$ is the radius of the neighborhood, and $N$ is used to prune out points far away from the simplex. For $N$, we typically take $N = \log(n)$ in theory and $N = 3$ in practice. Concerning $\Delta$, we use a heuristic choice $\Delta = \max_{i} \|Y_i-\bar{Y}\|/5$, where $\bar{Y} = \frac{1}{n}\sum_{i=1}^n Y_i$. It works satisfactorily in simulations.

{\bf Remark 3} {\it (P-SPA and D-SPA)}: We can view pp-SPA as a generic algorithm, where we may either replace the projection step by a different dimension reduction step, or replace the denoise step by a different denoise idea, or both.  In particular, it is interesting to consider two special cases: (i) {\it P-SPA}, which skips the denoise step and only uses the projection and VH steps; 
(ii) {\it D-SPA}, which skips the projection step and only uses the denoise and VH steps. 
We analyze these algorithms, together with pp-SPA (see Table~\ref{tb:order} and Section~\ref{sec:Proof-main} of the appendix). In this way, we can better  
understand the respective improvements of the projection step and the denoise step.

%
%
%

\section{An improved bound for SPA} \label{subsec:SPA}
Recall that $V = [v_1, v_2, \ldots, v_K]$, whose columns are the $K$ vertices of the true simplex ${\cal S}_0$. 
Let 
\begin{equation} \label{old-g-beta}
\gamma(V) = \max_{1 \leq k \leq K} \{\|v_k\|\},  \qquad g(V) = 1+80\frac{\gamma^2(V)}{s_K^2(V)}, \qquad \beta(X)  = \max_{1 \leq i \leq n} \{\|\eps_i\|\}.   
\end{equation}

\begin{lemma}[\cite{gillis2013fast}, orthodox SPA] \label{lemma:gill} 
Consider $d$-dimensional vectors $X_1, \ldots, X_n$, where $X_i = r_i + \eps_i$, $1 \leq i \leq n$ and $r_i$ satisfy model (\ref{model2}). 
For each $1 \leq k \leq K$ there is an $i$ such that $\pi_i=e_k$. Suppose  $\max_{1\leq i\leq n}\|\eps_i\|\leq\frac{s_K(V)}{1+80\gamma^2(V)/s_K^2(V)}\min\{\frac{1}{2\sqrt{K-1}}, \frac{1}{4}\}$. 
Apply the orthodox SPA to $X_1, \ldots, X_n$ and let $\hat{v}_1,\hat{v}_2,\ldots,\hat{v}_K$ be the output. Up to a permutation of these $K$ vectors,
\[
\max_{1\leq k\leq K}\{ \|\hat{v}_k - v_k\|\}\leq \Bigl[1+80\frac{\gamma^2(V)}{s_K^2(V)}\Bigr]\max_{1\leq i\leq n}\|\eps_i\| := g(V) \cdot \beta.   
\]
\end{lemma} 
Lemma \ref{lemma:gill} is among the best known results for SPA, but 
this bound is still not satisfying. 
One issue is that $s_K(V)$ depends on the location (i.e., center) of ${\cal S}_0$, but how well we can do vertex hunting should not depend on its location. We expect that vertex hunting is difficult only if ${\cal S}_0$ has a small volume (so the simplex is nearly flat). 
To see how these insights connect to singular values of $V$,    
let $\bar{v} =K^{-1}\sum_{k = 1}^K v_k$ be the center of ${\cal S}_0$, define $\tilde{V} = [v_1 - \bar{v},  \ldots, v_K - \bar{v}]$,  
and let $s_k(\tilde{V})$ be the $k$-th singular value of $\tilde{V}$.  
The next lemma is proved in the appendix:
\begin{lemma} \label{lemma:SPA} 
$\mathrm{Volume}({\cal S}_0)=\frac{\sqrt{K}}{(K-1)!}\prod_{k=1}^{K-1}s_{k}(\tilde{V})$, $s_{K-1}(V)\geq s_{K-1}(\tilde{V})$, and $s_K(V) \leq \sqrt{K} \|\bar{v}\|$.    
\end{lemma} 
Lemma \ref{lemma:SPA} yields several observations. First, as we shift the location of ${\cal S}_0$ so that its center gets close to the origin, $\|\bar{v}\|\approx 0$, and $s_K(V)\approx 0$. In this case, the bound in Lemma~\ref{lemma:gill} becomes almost useless. Second,  the volume of ${\cal S}_0$ is determined by the first $(K-1)$ singular values of $\tilde{V}$, irrelevant to the $K$th singular value. Finally, if the volume of ${\cal S}_0$ is lower bounded, then we immediately get a lower bound for $s_{K-1}(V)$. These observations motivate us to modify $g(V)$ in (\ref{old-g-beta}) to a new quantity that depends on $s_{K-1}(V)$ instead of $s_K(V)$; see (\ref{new-g-beta}) below.

\begin{wrapfigure}{r}{.4\linewidth}
\centering
        \includegraphics[width=.65\linewidth]{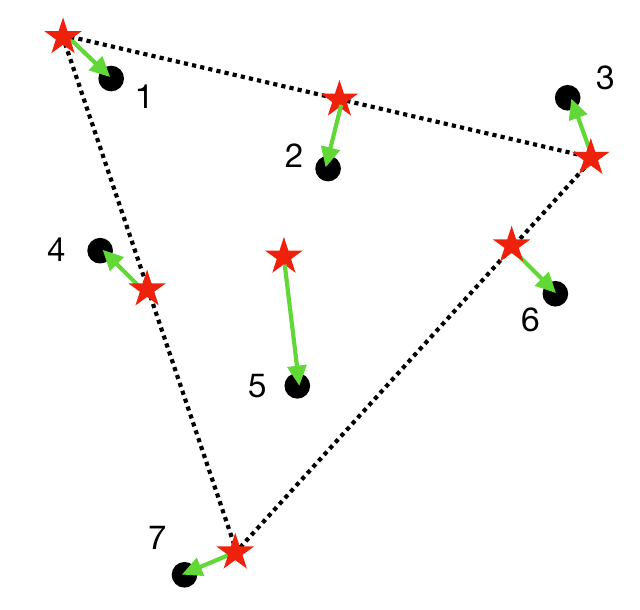}
        \caption{A toy example to show the difference between $\beta(X)$ and $\beta_{\text{new}}(X,V)$, where $\beta(X)=\max_i \|\epsilon_i\|$, and $\beta_{\text{new}}(X, V)\leq \max_{i\notin\{2,5\}}\|\epsilon_i\|$. } \label{fig:toy} 
    \end{wrapfigure}
Another issue of the bound in Lemma~\ref{lemma:gill} is that $\beta(X)$ depends on the maximum of $\|\epsilon_i\|$, which is too conservative. Consider a toy example in Figure~\ref{fig:toy}, where ${\cal S}_0$ is the dashed triangle, the red stars represent $r_i$'s and the black points are $X_i$'s. We observe that $X_2$ and $X_5$ are deeply in the interior of ${\cal S}_0$, and they should not affect the performance of SPA. We hope to modify $\beta(X)$ to a new quantity that does not depend on $\|\epsilon_2\|$ and $\|\epsilon_5\|$. 
One idea is to modify $\beta(X)$ to $\beta^*(X, V)=\max_i\mathrm{Dist}(X_i, \calS_0)$, where $\mathrm{Dist}(\cdot, \calS_0)$ is the Euclidean distance from a point to the simplex. For any point inside the simplex, this Euclidean distance is exactly zero. 
Hence, for this toy example, $\beta^*(X, V)\leq \max_{i\notin\{1,2,5\}}\|\epsilon_i\|$. 
However, we cannot simply replace $\beta(X)$ by $\beta^*(X, V)$, because $\|\epsilon_1\|$ also affects the performance of SPA and should not be left out. Note that $r_1$ is the only point located at the top vertex. When $X_1$ is far away from $r_1$, no matter whether $X_1$ is inside or outside $\calS_0$, SPA still makes a large error in estimating this vertex.  
This inspires us to define $\beta^{\dag}(X, V)=\max_{k} \min_{\{i:r_i=v_k\}}\|\epsilon_i\|$. When $\beta^{\dag}(X, V)$ is small, it means for each $v_k$, there exists at least one $X_i$ that is close enough to $v_k$. To this end, let $\beta_{\text{new}}(X,V)=\max\{\beta^*(X,V), \beta^{\dag}(X,V)\}$. Under this definition, $\beta_{\text{new}}(X)\leq \max_{i\notin\{2,5\}}\|\epsilon_i\|$, which is exactly as hoped.

Inspired by the above discussions, we introduce (for a point $x\in\mathbb{R}^d$, $\mathrm{Dist}(x, \calS_0)$ is the Euclidean distance from $x$ to $\calS_0$; this distance is zero if $x\in \calS_0$)
\begin{eqnarray} \label{new-g-beta}
g_{\mathrm{new}}(V) &=& 1+\frac{30\gamma(V)}{s_{K-1}(V)}\max\Bigl\{1, \frac{\gamma(V)}{s_{K-1}(V)}\Bigr\}, \cr 
\beta_{\mathrm{new}}(X) &=& \max \bigl\{  \max_{1\leq i\leq n}\mathrm{Dist}(X_i, \calS_0),\;   \max_{1\leq k\leq K} \min_{\{i:r_i=v_k\}}\|X_i-v_k\|\bigr\}. 
\end{eqnarray}
\begin{theorem} \label{thm:SPA-maintext} 
Consider $d$-dimensional vectors $X_1, \ldots, X_n$, where $X_i = r_i + \eps_i$, $1 \leq i \leq n$ and $r_i$ satisfy model (\ref{model2}). 
For each $1 \leq k \leq K$ there is an $i$ such that $\pi_i=e_k$. Suppose for a properly small universal constant $c^*>0$, 
$\max\{1, \frac{\gamma(V)}{\sigma_{K-1}(V)}\}\beta_{\mathrm{new}}(X,V) \leq c^*\frac{s^2_{K-1}(V)}{\gamma(V)}$.
Apply the orthodox SPA to $X_1, \ldots, X_n$ and let $\hat{v}_1,\hat{v}_2,\ldots,\hat{v}_K$ be the output.  Up to a permutation of these $K$ vectors,
\[
\max_{1\leq k\leq K}\{ \|\hat{v}_k - v_k\|\}\leq g_{\mathrm{new}}(V) \beta_{\mathrm{new}}(X, V). 
\]
\end{theorem}

Note that $g_{\mathrm{new}}(V)\leq g(V)$ and $\beta_{\mathrm{new}}(X,V)\leq \beta(X)$. The non-asymptotic bound in Theorem~\ref{thm:SPA-maintext} is always better than the bound in Lemma~\ref{lemma:gill}. We use an example to illustrate that the improvement can be substantial. Let $K=d=3$, $v_1=(20,20,10)$, $v_2 = (20,30,10)$, and $v_3 = (30,22,10)$. We put $r_1, r_2, r_3$ at each of the three vertices, $r_4, r_5, r_6$ at the mid-point of each edge, and $r_7$ at the center of the simplex (which is $\bar{v}$).  
We sample $\epsilon_1^*,\epsilon_2^*,\ldots,\epsilon_7^*$ i.i.d., from the unit sphere in $\mathbb{R}^3$. Let $\epsilon_i=0.01\epsilon_i^*$, for $1\leq i\leq 6$, and $\epsilon_7=0.05\epsilon_i^*$. By straightforward calculations, 
$g(V)=4.3025\times 10^4$, $g_{\mathrm{new}}(V)= 6.577 \times 10^2$, $\beta(X)= 0.05$, $\beta_{new}(X,V)=0.03$. 
Therefore, the bound in Lemma~\ref{lemma:gill} gives $\max_k \|\hat{v}_k - v_k\|\leq 2151.3$, while the improved bound in Theorem~\ref{thm:SPA-maintext} gives $\max_k \|\hat{v}_k - v_k\|\leq 18.7$. A more complicated version of this example can be found in Section~\ref{supp:Example} of the supplementary material.

The main reason we can achieve such a significant improvement is that our proof idea is completely different from the one in \cite{gillis2013fast}. The proof in \cite{gillis2013fast} is driven by {\it matrix norm inequalities} and does not use any geometry.  
This is why they need to rely on quantities such as $s_K(V)$ and $\max_{i}\|\epsilon_i\|$ to control the norms of various matrices in their analysis. It is very difficult to modify their proof to obtain Theorem~\ref{thm:SPA-maintext}, as the quantities in (\ref{new-g-beta}) are insufficient to provide strong matrix norm inequalities.  In contrast, our proof is guided by {\it geometric insights}. We construct a {\it simplicial neighborhood} near each true vertex and show that the estimate $\hat{v}_k$ in each step of SPA must fall into one of these simplicial neighborhoods.

\section{The bound for pp-SPA and its improvement over SPA}  \label{subsec:ppSPA}  
We focus on the orthodox SPA in 
Section \ref{subsec:SPA}. 
In this section, we show that we can further improve the bound significantly if we use   pp-SPA for vertex hunting.  
Recall that we have also introduced P-SPA and D-SPA in Section~\ref{sec:method} as simplified versions of pp-SPA. We establish error bounds for P-SPA, D-SPA, and pp-SPA, under the Gaussian noise assumption in (\ref{model1}).
%
A high-level summary is in Table~\ref{tb:order}. Recall that P-SPA, D-SPA, and pp-SPA all create pseudo-points and then feed them into SPA. Different ways of creating pseudo-points only affect the term $\beta_{\mathrm{new}}(X, V)$ in the bound in Theorem~\ref{thm:SPA-maintext}. Assuming that $g_{\mathrm{new}}(V)\geq C$, the order of $\beta_{\mathrm{new}}(X,V)$ fully captures the error bound. Table~\ref{tb:order} lists the sharp orders of $\beta_{\mathrm{new}}(X, V)$ (including the constant).

\begin{table}[h!]
\caption{The sharp orders of $\beta_{\mathrm{new}}(X, V)$ (settings: $K\geq 3$, $d$ satisfies (\ref{cond:Kd}), $s_{K-1}(V) >C$, and $m$ satisfies the condition in Theorem~\ref{thm:main2}). P-SPA and D-SPA use {\it the projection only} and {\it the denoise only}, respectively. The constant $c_0\in (0,1)$ comes from $m$, and the constant $a_1>2$ is as in Lemma~\ref{lemma:chi2}. } \label{tb:order}
\centering 
\scalebox{0.9}{
\begin{tabular}{l | ccccccccc}
\hline
& $d\ll \log(n)$ & $d= a_0\log (n)$ & $\log(n)\ll d\ll n^{1-\frac{2(1-c_0)}{K-1}}$ & $d\gg n^{1-\frac{2(1-c_0)}{K-1}}$\\
\hline
SPA & $\sqrt{2\log(n)}$ & $\sqrt{a_1\log (n)}$ & $\sqrt{d}$ & $\sqrt{d}$\\
P-SPA & $\sqrt{2\log(n)}$ & $\sqrt{2\log(n)}$ &  $\sqrt{2\log(n)}$ & $\sqrt{2\log(n)}$\\
D-SPA & $\sqrt{2c_0\log(n)}$ & NA & NA & NA\\
pp-SPA&  $\sqrt{2c_0\log(n)}$ &$\sqrt{2c_0\log(n)}$ &  $\sqrt{2c_0\log(n)}$ & $\sqrt{2\log(n)}$ \\
\hline
\end{tabular}}
\end{table}


The results suggest that pp-SPA always has a strictly better error bound than SPA. When $d\gg\log(n)$, the improvement is a factor of $o(1)$; the larger $d$, the more improvement. When $d = O(\log(n))$, the improvement is a constant factor that is strictly smaller than $1$. In addition, by comparing P-SPA and D-SPA with SPA, we have some interesting observations: 
\begin{itemize}
\item {\it The projection effect}. From the first two rows of Table~\ref{tb:order}, the error bound of P-SPA is never worse than that of SPA. In many cases,  P-SPA leads to a significant improvement. When $d\gg \log(n)$, the rate is faster by a factor of $\sqrt{\log(n)/d}$ (which is a huge improvement for high-dimensional data). When $d\asymp \log(n)$, there is still a constant factor of improvement. 
\item {\it The denoise effect}. We compare the error bounds for P-SPA and pp-SPA, where the difference is caused by the denoise step. 
In three out of the four cases of $d$ in Table~\ref{tb:order}, pp-SPA strictly improves P-SPA by a constant factor $c_0<1$.  

We note that pp-SPA applies denoise to the projected data in $\mathbb{R}^{K-1}$. We may also apply denoise to the original data in $\mathbb{R}^d$, which gives D-SPA. By Table~\ref{tb:order}, when $d\ll \sqrt{\log(n)}$, D-SPA improves SPA by a constant factor. However, for $d\gg \log(n)$, we always recommend applying denoise to the projected data. In such cases, the leading term in the extreme value of chi-square (see Lemma~\ref{lemma:chi2}) is $d$, so the denoise is not effective if applied to original data.    
\end{itemize}  
Table~\ref{tb:order} and the above discussions are for general settings. In a slightly more restrictive setting (see Theorem~\ref{thm:main1} below), both projection and denoise can improve the error bounds by a factor of $o(1)$.

We now present the rigorous statements. Owing to space constraint, we only state the error bounds of pp-SPA in the main text. The error bounds of P-SPA and D-SPA can be found in the appendix.  

\subsection{Some useful preliminary results} 
\label{subsec:prelim} 
Recall that $V = [v_1, \ldots, v_K]$ and $r_i  = V \pi_i$, $1 \leq i \leq n$. Let $\bar{v}$, $\bar{r}$, and $\bar{\pi}$ be the empirical means of $v_k$'s, $r_i$'s,  and $\pi_i$'s, respectively.   Introduce 
$\tilde{V} =  [v_1 - \bar{v},  \ldots, v_K - \bar{v}]$, $R = n^{-1/2} [r_1 - \bar{r},  \ldots, r_n - \bar{r}]$, and $G = (1/n) \sum_{i = 1}^n (\pi_i - \bar{\pi}) (\pi_i - \bar{\pi})'$. 
Lemma \ref{lemma:G} relates singular values of $R$ to those of $G$ and $V$ and is proved in the appendix ($A \preceq B$: $B-A$ is positive  semi-definite. Also,  $\lambda_k(G)$ is the $k$-th largest (absolute value)  eigenvalue of $G$, $s_k(V)$ is the $k$-th largest singular value of $V$; same below).  
\begin{lemma} \label{lemma:G} 
The following statements are true: (a) $RR' = VGV'$,  (b) $\lambda_{K-1}(G) \cdot \tilde{V} \tilde{V}' \preceq V G V' \preceq \lambda_1(G)\cdot \tilde{V} \tilde{V}'$, and (c) $\lambda_{K-1}(G)\cdot  s_{K-1}^2(\tilde{V}) \preceq \sigma_{K-1}^2(R)  \preceq  \lambda_1(G) \cdot s_{K-1}^2(\tilde{V})$.  
\end{lemma} 

To analyze SPA and pp-SPA,  we need precise results on the extreme values of chi-square variables. Lemma \ref{lemma:chi2} is proved in the appendix.  
\begin{lemma} \label{lemma:chi2} 
Let $M_n$ be the maximum of $n$ $iid$ samples from $\chi_d^2(0)$.  As $n \goto \infty$, (a) if $d\ll \log (n)$, then $M_n /(2 \log(n)) \goto 1$, (b) if $d \gg \log(n)$, then $M_n / d \goto 1$, and (c) if $d=a_0\log (n)$ for a constant  $a_0>0$, then $M_n / (a_1 \log (n) ) \goto 1$ where $a_1>2$ is unique solution of the equation $a_1 - a_0 \log (a_1) = 2 + a_0  - a_0 \log (a_0)$ 
(convergence in three cases are convergence  in probability). 
\end{lemma}

\subsection{Regularity conditions and main theorems}

We assume 
\begin{equation}  \label{cond:Kd} 
 K = o(\log(n)/\log\log(n)), \qquad d = o(\sqrt{n}). 
\end{equation}  
These are mild conditions. In fact, in practice, the dimension of the 
true simplex is usually relatively low, so the first condition is 
mild. Also, when the (low-dimensional) true simplex 
is embedded in a high dimensional space, it is not preferable 
to directly apply vertex hunting. Instead, one would use tools such as 
PCA to significantly reduce the dimension first and then 
perform vertex hunting. For this reason, the second condition is also mild. 
Moreover, recall that $G = n^{-1} \sum_{i = 1}^n (\pi_i - \bar{\pi})(\pi_i - \bar{\pi})'$ 
is the empirical covariance matrix of the (weight vector) $\pi_i$ and 
$\gamma(V) = \max_{1 \leq k \leq K} \{\|v_k\|\}$. We assume for some constant $C>0$, 
\begin{equation} \label{cond:G}  
\lambda_{K-1}(G) \geq C^{-1}, \qquad \lambda_1(G) \leq C, \qquad \gamma(V) \leq C. 
\end{equation} 
The first two items are a mild balance condition on $\pi_i$ and the last one is a natural condition on $V$. 
Finally, in order for the (orthodox) SPA to perform well, we need 
\begin{equation} \label{cond:SPA2} 
\sigma \sqrt{\log(n)} / s_{K-1}(\tilde{V}) \goto 0. 
\end{equation} 
In many applications, vertex hunting is used as a module in the main algorithm, and the data points fed into VH are from previous steps of some algorithm and satisfy $\sigma=o(1)$ (for example, see \cite{MSCORE, ke2017new}). Hence, this condition is reasonable. 




We present the main theorems (which are used to obtain Table~\ref{tb:order}). In what follows, Theorem~\ref{thm:main2} is for a general setting, and Theorem~\ref{thm:main1} concerns a slightly more restrictive setting. For each setting, we will specify explicitly the theoretically optimal choices of thresholds $(t_n, \epsilon_n)$ in pp-SPA.

For $1 \leq k \leq K$, let $J_k=\{i: r_i=v_k\}$ be the set of $r_i$ located at vertex $v_k$, and let $n_k=|J_k|$, for $1\leq k\leq K$. Let $\Gamma(\cdot)$ denote the standard Gamma function. Define 
\begin{equation}\label{Definec2}
m = \min\{n_1, n_2, \ldots, n_K\}, \qquad c_2 = 0.5 (2e^{2})^{-\frac{1}{K-1}} \sqrt{2/(K-1)} \bigl[\Gamma (\frac{K+1}{2}) \bigr]^{\frac{1}{K-1}}. 
\end{equation}
Note that as $K\to\infty$, $c_2 \goto 0.5 / \sqrt{e}$. We also introduce
\begin{equation} \label{Definemalphan} 
\alpha_n  = \frac{\sqrt d}{\sqrt n s^2_{K-1}(\tilde{V})}  \bigl(1+   \sigma  \sqrt{\max\{d, 2\log (n)\}}\bigr), \qquad  b_n=  \frac{2 \sigma }{\sqrt n} \sqrt{\max\{d, 2\log (n)\}}.   
\end{equation} 
The following theorem is proved in the appendix. 
\begin{theorem}  \label{thm:main1}  
Suppose $X_1, X_2, \ldots, X_n$ are generated from model (\ref{model1})-(\ref{model2}) 
where $m \geq  c_1 n$ for a constant $c_1 > 0$ and conditions (\ref{cond:Kd})-(\ref{cond:SPA2}) hold.  Fix $\delta_n$ such that $ (K-1)/\log(n)  \ll \delta_n \ll 1$, and let 
$t_n = \sqrt{K-1} \bigl(\frac{\log(n)}{n^{1- \delta_n}}\bigr)^{\frac{1}{K-1}}$. 
We apply pp-SPA to $X_1, X_2, \ldots, X_n$ with $(N, \Delta)$ to be determined below.  Let $\hat{V} = [\hat{v}_1, \hat{v}_2, \ldots, \hat{v}_K]$, where $\hat{v}_1, \hat{v}_2, \ldots, 
\hat{v}_K$ are the estimated vertices.  
\begin{itemize} 
\item In the first case, $ \alpha_n \ll  t_n$. We take $N = \log(n)$ and 
$\Delta = c_3 t_n \sigma$ in pp-SPA, for a constant $c_3 \leq c_2$. Up to a permutation of $\hat{v}_1,\ldots,\hat{v}_K$, 
$\max_{1 \leq k \leq K} \{\|\hat{v}_k - v_k\|\}\leq  \sigma g_{\mathrm{new}}(V) [ \sqrt{\delta_n}\cdot   \sqrt{ 2\log (n)} + C \alpha_n] + b_n$.  
\item In the second case, $ t_n\ll  \alpha_n \ll 1$. We take $N = \log(n)$ and 
$\Delta = \sigma  \alpha_n$ in pp-SPA. Up to a permutation of $\hat{v}_1,\ldots,\hat{v}_K$, $
\max_{1 \leq k \leq K} \{\|\hat{v}_k - v_k\|\}\leq  \sigma g_{\mathrm{new}}(V) \cdot (1+ o_{\mathbb{P}}(1))  \sqrt{ 2\log (n) }$. 
\end{itemize} 
\end{theorem}

To interpret Theorem~\ref{thm:main1}, we consider a special case where $K=O(1)$, $s_{K-1}(\tilde{V})$ is lower bounded by a constant, and we set $\delta_n=\log\log(n)/\log(n)$.  
By our assumption (\ref{cond:Kd}), $d=o(\sqrt{n})$. It follows that
$\alpha_n\asymp \max\bigl\{d, \sqrt{d\log(n)}\bigr\} /\sqrt{n}$, $b_n\asymp \sigma \sqrt{\max\{d,\, \log(n)\}/n}$, and $t_n\asymp [\log(n)]^{\frac{1}{K-1}}/n^{\frac{1-o(1)}{K-1}}$. 
We observe that $\alpha_n$ always dominates $b_n/\sigma$. 
Whether $\alpha_n$ dominates $t_n$ is determined by $d/n$. When $d/n$ is properly small so that $\alpha_n\ll  t_n$, using the first case in Theorem~\ref{thm:main1}, we get $
\max_k \{\|\hat{v}_k - v_k\|\}\leq  C \bigl(\sqrt{\log(\log(n))} + \max\bigl\{d, \sqrt{d\log(n)}\bigr\} /\sqrt{n}\bigr)=O(\sqrt{\log\log(n)})$. 
When $d/n$ is properly large so that $\alpha_n\gg t_n$, using the second case in Theorem~\ref{thm:main1}, we get $
\max_k \{\|\hat{v}_k - v_k\|\}=O\bigl(\sqrt{\log(n)}\bigr)$. 
We then combine these two cases and further plug in the constants in Theorem~\ref{thm:main1}. It yields
\begin{equation} \label{rate-ppspa}
\max_{1 \leq k \leq K} \{\|\hat{v}^{\text{ppspa}}_k - v_k\|\} \leq  \sigma g_{\mathrm{new}}(V)\cdot  \left\{
\begin{array}{ll}
\sqrt{\log\log(n)} & \text{ if $d/n$ is properly small};\\
\sqrt{[2+o(1)] \log (n)} & \text{ if $d/n$ is properly large}.
\end{array}
\right.
\end{equation}

%
%
It is worth comparing the error bound in Theorem~\ref{thm:main1} with that of the orthodox SPA (where we directly apply SPA on the original data points $X_1, X_2, \ldots, X_n$). Recall that $\beta(X)$ is as defined in (\ref{new-g-beta}). Note that $\beta(X)\leq \max_{1\leq i\leq n}\|\epsilon_i\|$, where $\|\epsilon_i\|^2$ are i.i.d. variables from $\chi^2_d(0)$. Combining Lemma~\ref{lemma:chi2} and Theorem~\ref{thm:SPA-maintext}, we immediately obtain that for the (orthodox) SPA estimates $\hat{v}^{spa}_1, \hat{v}^{spa}_2, \ldots, \hat{v}^{spa}_K$, up to a permutation of these vectors (the constant  $a_1$ is as in Lemma~\ref{lemma:chi2} and satisfies $a_1>2$):
\begin{equation} \label{rate-oSPA}
\max_{1 \leq k \leq K} \{\|\hat{v}^{\text{spa}}_k - v_k\|\} \leq  \sigma g_{\mathrm{new}}(V)\cdot  \left\{
\begin{array}{ll}
\sqrt{\max\{d, \; 2 \log(n)\}} & \text{ if $d\ll \log(n)$ or $d\gg \log (n)$};\\
\sqrt{a_1 \log (n)} & \text{ if $d= a_0 \log(n)$}.
\end{array}
\right.
\end{equation}
This bound is tight (e.g., when all $r_i$ fall into vertices). We compare (\ref{rate-oSPA}) with Theorem~\ref{thm:main1}. If $d\gg \log(n)$, the improvement is a factor of $\sqrt{\log(n)/d}$, which is huge when $d$ is large. If $d=O(\log(n))$, the improvement can still be a factor of $o(1)$ sometimes (e.g., in the first case of Theorem~\ref{thm:main1}). 

Theorem~\ref{thm:main1} assumes that there are a constant fraction of $r_i$ falling at each vertex. This can be greatly relaxed. 
The following theorem is proved in the appendix.  
\begin{theorem} \label{thm:main2} 
Fix $0 < c_0  < 1$ and a sufficiently small constant $0 < \delta < c_0$.  
Suppose $X_1, X_2, \ldots, X_n$ are generated from model (\ref{model1})-(\ref{model2}) 
where $m \geq  n^{1-c_0 + \delta}$ and conditions (\ref{cond:Kd})-(\ref{cond:SPA2}) hold.  Let $
t_n^* = \sqrt{K-1}  \bigl(\frac{\log(n)}{n^{1 - c_0}}\bigr)^{\frac{1}{K-1}}$.  
We apply pp-SPA to $X_1, X_2, \ldots, X_n$ with $(N, \Delta)$ to be determined below.  Let $\hat{V} = [\hat{v}_1, \hat{v}_2, \ldots, \hat{v}_K]$, where $\hat{v}_1, \hat{v}_2, \ldots, 
\hat{v}_K$ are the estimated vertices.  
\begin{itemize} 
\item In the first case, $\alpha_n \ll t_n^*$. We take $N = \log(n)$ and 
$\Delta = c_3 t_n \sigma$ in pp-SPA, for a constant $c_3 \leq e^{c_0/(K-1)} c_2$. Up to a permutation of $\hat{v}_1, \ldots,\hat{v}_K$, 
$\max_{1 \leq k \leq K} \{\|\hat{v}_k - v_k\|\} \leq  \sigma g_{\mathrm{new}}(V) [ \sqrt{c_0}\cdot    \sqrt{ 2\log (n)} + C \alpha_n] + b_n$. 
\item In the second case,  $ \alpha_n \gg t_n^*$.  Suppose $\alpha_n = o(1)$. 
We take $N = \log(n)$ and $\Delta = \alpha_n$ in pp-SPA.  
Up to a permutation of $\hat{v}_1, \ldots,\hat{v}_K$, $
\max_{1 \leq k \leq K} \{\|\hat{v}_k - v_k\|\} \leq  \sigma g_{\mathrm{new}}(V) \cdot (1+ o_{\mathbb{P}}(1))  \sqrt{ 2\log (n) }$. 
\end{itemize} 
\end{theorem}

Comparing Theorem~\ref{thm:main2} with Theorem~\ref{thm:main1}, the difference is in the first case, where the $o(1)$ factor of $\delta_n$ is replaced by a constant factor of $c_0<1$. Similarly as in (\ref{rate-ppspa}), we obtain
\begin{equation} \label{rate-ppspa2}
\max_{1 \leq k \leq K} \{\|\hat{v}^{\text{ppspa}}_k - v_k\|\} \leq  \sigma g_{\mathrm{new}}(V)\cdot  \left\{
\begin{array}{ll}
\sqrt{2c_0\log(n)} & \text{ if $d/n$ is properly small};\\
\sqrt{[2+o(1)] \log (n)} & \text{ if $d/n$ is properly large}.
\end{array}
\right.
\end{equation}

In this relaxed setting, we also compare Theorem~\ref{thm:main2} with (\ref{rate-oSPA}): (a) When $d\gg \log(n)$, the improvement is a factor of $\sqrt{\log(n)/d}$. (b) When $d=O(\log(n))$, the improvement is at the constant order. It is interesting to further compare these ``constants". Note that $g_{\mathrm{new}}(V)$  is the same for all methods. It suffices to compare the constants in the bound for $\beta_{\mathrm{new}}(V)$. 
In Case (b), the error bound of pp-SPA is smaller than that of SPA by a factor of $c_0\in (0,1)$. For the practical purpose, even the improvement of a constant factor can have a huge impact, especially when the data contain strong noise and potential outliers. Our simulations in Section~\ref{sec:Simu} further confirm this point.

\section{Numerical study} \label{sec:Simu}

We compare SPA, pp-SPA, and two simplified versions P-SPA and D-SPA (for illustration). We also compared these approaches with robust-SPA (\citealp{gillis2019successive}) from \texttt{bit.ly/robustSPA} (with default tuning parameters). For pp-SPA and D-SPA, we need to specify tuning parameters $(N,\Delta)$. We use the heuristic choice in Remark~2. Fix $K=3$ and three points $\{y_1,y_2,y_3\}$ in $\mathbb{R}^2$. Given $(n, d, \sigma)$, we first draw $(n-30)$ points uniformly from the $2$-dimensional simplex whose vertices are $y_1,y_2,y_3$, and then put $10$ points on each vertex of this simplex. Denote these points by $w_1,w_2,\ldots,w_n\in\mathbb{R}^2$. Next, we fix a matrix $A\in\mathbb{R}^{d\times 2}$, whose top $2\times 2$ block is equal to $I_d$ and the remaining entries are zero. Let $r_i=Aw_i$, for all $i$. Finally, we generate $X_1,X_2,\ldots,X_n$ from model~(\ref{model1}). We consider three experiments. In Experiment~1, we fix $(n,\sigma)=(1000,1)$ and let $d$ range in $\{1,2,\ldots,49, 50\}$. In Experiment~2, we fix $(n, d)=(1000,4)$ and let $\sigma$ range in $\{0.2, 0.3, \ldots, 2\}$. In Experiment~3, we fix $(d, \sigma)=(4, 1)$ and let $n$ range in $\{500, 600, \ldots, 1500\}$. We evaluate the vertex hunting error $\max_{k}\{\|\hat{v}_k-v_k\|\}$ (subject to a permutation of $\hat{v}_1,\ldots,\hat{v}_K$).  
For each set of parameters, we report the average error over $20$ repetitions. The results are in Figure~\ref{fig:exp}. 
They are consistent with our theoretical insights: The performances of P-SPA and D-SPA are both better than that of SPA, and the performance of pp-SPA is better than those of P-SPA and D-SPA. It suggests that both the projection and denoise steps are effective in reducing noise, and it is beneficial to combine them. When $d\leq 10$, pp-SPA, P-SPA and D-SPA all outperform robust-SPA; when $d>10$, both pp-SPA and P-SPA outperform robust-SPA, and D-SPA (the simplified version without hyperplain projection) underperforms robust-SPA. The code to reproduce these experiments is available at \url{https://github.com/Gabriel78110/VertexHunting}.

\begin{figure}[tb!]
    \centering
    \includegraphics[width=\textwidth, height=1.4in, trim=0 10 0 0, clip=true]{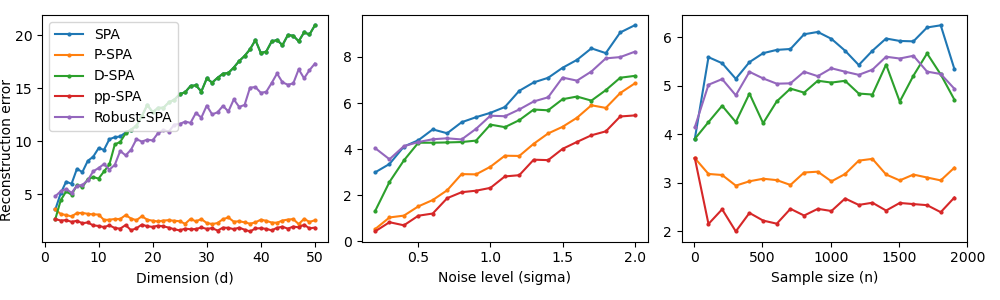}
    \caption{Performances of SPA, P-SPA, D-SPA, and pp-SPA in Experiment~1-3.}
    \label{fig:exp}
\end{figure}
%

\section{Discussion}

Vertex hunting is a fundamental problem found in many
applications. The Successive Projection algorithm (SPA) 
is a popular approach, but may behave 
unsatisfactorily in many settings.   We propose pp-SPA as a new 
approach to vertex hunting. Compared to SPA, the new algorithm provides much improved theoretical bounds and encouraging improvements in a wide variety of numerical study.  We also provide a sharper non-asymptotic bound for the orthodox SPA.  For technical simplicity, our model assumes Gaussian noise, but our results are readily extendable to 
subGaussian noise. Also, our non-asymptotic bounds do not require any distributional assumption, and are directly applicable to different settings. 
For future work, we note that an improved bound on vertex hunting frequently 
implies improved bounds for methods that {\it contains vertex hunting as an important step}, such as Mixed-SCORE for network analysis \citep{MSCORE,bhattacharya2023inferences},    Topic-SCORE for 
text analysis \citep{ke2017new}, and state compression of Markov processes \citep{zhang2019spectral}, where vertex hunting plays a key role.
Our algorithm and bounds may also be useful for related problems such as estimation of convex density support \citep{brunel2016adaptive}.

%
%
%

\newcommand{\beq}{\begin{equation}}
\newcommand{\eeq}{\end{equation}}

\appendix
\renewcommand{\thesection}{\Alph{section}}
\renewcommand{\theequation}{\Alph{section}.\arabic{equation}}
\renewcommand*{\thetheorem}{\Alph{section}.\arabic{theorem}}
\renewcommand*{\thelemma}{\Alph{section}.\arabic{lemma}}

\newtheorem{thm}{Theorem}[section]
\newtheorem{lem}{Lemma}[section]

\setcounter{equation}{0}

\section{Proof of preliminary lemmas}

\subsection{Proof of Lemma~\ref{lemma:projection}}
This is a quite standard result, which can be found at tutorial materials (e.g.,  \url{https://people.math.wisc.edu/~roch/mmids/roch-mmids-llssvd-6svd.pdf}). We include a proof here only for convenience of readers.

We start by introducing some notation. Let $Z_i=X_i-\bar X$ and let $Z=[Z_1,\ldots,Z_n]\in\mathbb R^{d,n}$. Suppose the singular value decomposition of Z is given by $Z=U_ZD_ZV_Z'$. Since $H$ is a rank-$(K-1)$ projection matrix, we have $H=QQ'$, where $Q\in\mathbb R^{d,K-1}$ is such that $Q'Q=I_{K-1}$. Hence, we rewrite the optimization in (\ref{hplane1}) as follows:
\begin{align*}
\mbox{minimize }\sum_{i=1}^n(X_i-x_0)'(I_d-QQ')(X_i-x_0),\quad \mbox{subject to}\quad Q'Q=I_{K-1}.
\end{align*}
For $\lambda\in\mathbb R$, consider the Lagrangian objective function
\begin{align}\label{objectivefunciton}
\widetilde S(x_0,Q,\lambda)=\sum_{i=1}^n(X_i-x_0)'(I_d-QQ')(X_i-x_0)+\lambda(Q'Q-I_{K-1}).
\end{align}
Setting its gradients w.r.t.~$x_0$ and $Q$ to be 0 yields
\begin{align}
&\nabla_{x_0}\widetilde S(x_0,Q,\lambda)=-2(I_d-QQ')\sum_{i=1}^n(X_i-x_0)=0,\label{grad1}\\
&\nabla_{Q}\widetilde S(x_0,Q,\lambda)=-2Q'\sum_{i=1}^n(X_i-x_0)(X_i-x_0)'+2\lambda Q'=0.\label{grad2}
\end{align}
Firstly, we deduce from (\ref{grad1}) that $\hat x_0=\bar X$, which in view of (\ref{grad2}) implies that $Q'(ZZ'-\lambda I_d)=0$.
 The above equations also implies that the $(K-1)$ columns of $\widehat Q$ should be the distinct columns of $U_Z$.  Now, the objective function in (\ref{objectivefunciton}) is given by
\begin{align}\label{tildeS}
\widetilde S(x_0,Q,\lambda)&=\sum_{i=1}^nZ_i'(I_d-QQ')Z_i={\rm tr}[(I_d-QQ')ZZ']={\rm tr}[(I_d-QQ')U_ZD_Z^2U_Z']\notag\\
&={\rm tr}(D_Z)^2-{\rm tr}[Q'U_ZD_Z^2U_Z'Q]={\rm tr}(D_Z^2)-\|D_ZU_Z'Q\|_{\rm F}^2.
\end{align}
Note that for each column of $U_Z'Q\in\mathbb R^{d,K-1}$, it has exactly one entry being 1 and its other entries are all 0. Therefore, taking $\widehat Q=U$ maximizes $\|D_ZU_Z'Q\|_{\rm F}^2$ and hence minimizes the objective function $\widetilde S$ in (\ref{objectivefunciton}), that is, $\widehat H=UU'$. The proof is complete.

\subsection{Proof of Lemma~\ref{lemma:SPA} }
For the simplex formed by  $V\in \mathbb{R}^{d \times K}$, we can always find an orthogonal matrix $O\in\mathbb{R}^{d\times d}$ and a scalar $a$ such that 
\begin{align*}
OV= \begin{pmatrix} x_1 & x_2&\ldots & x_K\\ a& a&\ldots & a\\ 0&0 & \ldots&  0 \end{pmatrix}, \quad \text{ where } \quad  x_k\in \mathbb{R}^{K-1} \text{ for } k=1, \ldots, K.
\end{align*}
Denote $\bar x = K^{-1} \sum_{k=1}^K x_k$. Further we can represent 
\begin{align*}
O  \tilde V= \begin{pmatrix}  x_1  -\bar x & x_2 - \bar x&\ldots & x_K - \bar x\\ 0&0&\ldots& 0 \end{pmatrix}
\end{align*}
We write $\tilde X: = (x_1 - \bar x, x_2 - \bar x, \ldots, x_K - \bar x)$.
 Since rotation and location do not change the volume, 
 \begin{align*}
{\rm Volume}(\mathcal S_0)   = {\rm Volume}(\mathcal S(\tilde X)).    
\end{align*}
 where $\mathcal S(\tilde X) $ represents the simplex formed by $\tilde X$. 
By \cite{stein1966note}, we have 
\begin{align*}
{\rm Volume}(\mathcal S_0)   = \frac{{\rm det}( \tilde  A) }{(K-1)!}\, , \quad\text{ with } \quad 
\tilde A = \left[
\begin{array}{cc}
1 & (x_1- \bar x)'   \\
1  &  (x_2- \bar x)'  \\
\vdots & \vdots \\ 
1 &(x_K- \bar x)' \\
\end{array} 
\right] 
\end{align*}
We also define 
\begin{align*}
A =  \left[
\begin{array}{cc}
1 &  (v_1 - \bar v)'   \\
1  &  (v_2 - \bar v)'  \\
\vdots & \vdots \\ 
1 &  (v_K - \bar v)' \\ 
\end{array} 
\right] = [{\bf 1}_K , \tilde V'],  
\end{align*}
Since $(\tilde A, 0 ) = A\begin{pmatrix} 1& 0\\ 0 & O\end{pmatrix} $, it follows that $\tilde{A}\tilde{A}' = AA'$ and ${\rm Volume}(\mathcal S_0)   = \frac{\sqrt{{\rm det}( AA') } }{(K-1)!} = \frac{\sqrt{{\rm det}( A' A) } }{(K-1)!}$.
Note that $A' A = \begin{pmatrix} K & 0\\ 0 & \tilde V \tilde V'\end{pmatrix}$ by the fact that $\tilde V {\bf 1}_K= 0$. Then ${\rm det} (A'A)  = K {\rm det} (\tilde V\tilde V' ).$ Further notice that ${\rm rank }(\tilde V \tilde V' ) = K-1$. 
We thus conclude that 
\begin{align*}
{\rm Volume}(\mathcal S_0)  =  \frac{\sqrt K}{(K-1)!} \prod_{k=1}^{K-1} s_k(\tilde V). 
\end{align*}
This proves the first claim. 

For the second and last claims, we first notice that $V = \tilde V - \bar v{\bf 1}_K'  $. Then $VV' = \tilde V\tilde V'+ K \bar v\bar v'   $ again by $\tilde V {\bf 1}_K= 0$. Because both $ \tilde V\tilde V'$ and $K \bar v\bar v' $ are positive semi-definite,  by Weyl's inequality (see,  for example \cite{HornJohnson}),   it follows that $s_{K-1} (V)\geq s_{K-1} (\tilde V)$ and $s_K(V) = \sqrt{\lambda_{\min}(VV')} \leq \sqrt{K \Vert \bar v\Vert^2 } = \sqrt K \Vert \bar v\Vert$.


\subsection{Proof of Lemma~\ref{lemma:G} }
We first prove claim (a). Let $\Pi=[\pi_1-\bar\pi,\ldots,\pi_n-\bar\pi]\in\mathbb R^{K,n}$. Recalling the definitions of $G$ and $V$, we have $G=n^{-1}\Pi\Pi'$ and $R=n^{-1/2}V\Pi$, so that $RR'=n^{-1}V\Pi\Pi'V'=VGV'$. 

Next, we prove claim (b). Recall that $\tilde V=V-\bar v1_K'$, so that $\tilde V\tilde V'=(V-\bar v1_K')(V-\bar v1_K')'=VV'-K\bar v\bar v'$. Note that
Since $\pi_i'1_K=\bar\pi'1_K=1$, we have $\Pi'1_K=0$, which implies that $G1_K=n^{-1}\Pi(\Pi'1_K)=0$. We deduce from this observation that $\lambda_K(G)=0$ and its associated eigenvector is $K^{-1/2}{\bf1}_K$. Therefore, $G-\lambda_{K-1}(G)I_K + K^{-1}\lambda_{K-1}(G){\bf1}_K{\bf1}_K' $ is a positive semi-definite matrix, so that
\begin{align*}
VGV'-\lambda_{K-1}(G)\tilde V\tilde V'&=VGV'-\lambda_{K-1}(G)VV'+\lambda_{K-1}(G)K\bar v\bar v'\\
&=V[G-\lambda_{K-1}(G)I_K +K^{-1} \lambda_{K-1}(G){\bf1}_K{\bf1}_K' ]V'\geq0.
\end{align*}
In addition, observing that $\Pi'1_K=0$ due to the fact that $\|\pi_i\|_1=\|\bar\pi\|_1=1$, we obtain that
\begin{align*}
\tilde VG\tilde V'=(V-\bar v1_K')G(V-\bar v1_K')'=n^{-1}(V-\bar v1_K')\Pi\Pi'(V-\bar v1_K')'=VGV'.
\end{align*}
Therefore,
\begin{align*}
\lambda_1(G)\tilde V\tilde V'-VGV'=\lambda_1(G)\tilde V\tilde V'-\tilde VG\tilde V'=\tilde V[\lambda_1(G)I_K-G]\tilde V'\geq0,
\end{align*}
which completes the proof of claim (b). 

Finally, for claim (c), we obtain from (a) that $\sigma_{K-1}^2(R)=\lambda_{K-1}(RR')=\lambda_{K-1}(VGV')$, which by Weyl's inequality (see, for example, \cite{HornJohnson}) and in view of claim (b) implies that $\lambda_{K-1}(G)\lambda_{K-1}(\tilde V\tilde V')\leq \sigma_{K-1}^2(R)\leq \lambda_{1}(G)\lambda_{K-1}(\tilde V\tilde V')$. The proof is therefore complete.

\subsection{Proof of Lemma \ref{lemma:chi2}} 
Recall that $z_1 \sim \chi_d^2(0)$. Let $b_n$ be the value such that 
\[
\mathbb{P}(z_1 \geq b_n) = 1/n.
\] 
By basic extreme value theory, it is known that 
\[
\frac{\max_{1 \leq i \leq n} \{z_i\}}{b_n}  \goto 1, \qquad \mbox{in probability}. 
\]

We now solve for $b_n$.  It is seen that $b_n \geq  d$.  
Recall that the density of $\chi_d^2(0)$ is 
\[
\frac{1}{2^{d/2} \Gamma(d/2)} x^{d/2-1} e^{-x/2},\qquad x > 0. 
\] 
Note that for any $x_0 \geq d$, 
\begin{equation} 
\int_{x_0}^{\infty} x^{d/2 - 1} e^{-x/2} dx = 2 x_0^{d/2-1} e^{-x_0/2} + \int_{x_0}^{\infty}  (d-2) x^{d/2-2} e^{-x/2} dx 
\end{equation} 
where the RHS is no greater than 
\[ 
\leq 2 x_0^{d/2-1} e^{-x_0/2}  + \frac{(d-2)}{x_0} \int_{x_0}^{\infty}  x^{d/2 - 1} e^{-x/2} dx. 
\] 
It follows that for all $x_0 \geq d$, 
\begin{equation} \label{Mill} 
2 x_0^{d/2-1} e^{-x_0/2} \leq \int_{x_0}^{\infty}  x^{d/2 - 1} e^{-x/2} dx  \leq   x_0 \cdot  x_0^{d/2-1} e^{-x_0/2},   
\end{equation} 
where we have used 
\[
\frac{x_0}{x_0 -d + 2}  \leq x_0/2. 
\] 
It now follows that there is a term $a(x)$ such that when $x \geq d$, 
\[
1 \leq a(x) \leq x/2  
\] 
and
\[ 
\mathbb{P}(z_1 \geq x) = a(x)  \frac{1}{2^{d/2} \gamma(d/2)} 2 x^{d/2-1} e^{-x/2}. 
\] 
Combining these, $b_n$ is the solution of 
\begin{equation} \label{equation1} 
a(x)  \frac{1}{2^{d/2} \gamma(d/2)} 2 x^{d/2-1} e^{-x/2} = \frac{1}{n}. 
\end{equation} 

We now solve the equation in (\ref{equation1}). Consider the case $d$ is even. 
The case where $d$ is odd is similar, so we omit it. When $d$ is even, using 
\[
\Gamma(d/2) = (d/2-1)! = (2/d) (d/2)! = (2/d) \theta (\frac{d}{2e})^{d/2}, 
\] 
where $\theta$ is the factor in the Stirling's formula which is $\leq C \sqrt{\log(d)}$. 
Plugging this into the left hand side of (\ref{equation1}) and re-arrange, we have 
\begin{equation} \label{equation2} 
\log(d/x) + (d/2) \log(\frac{ex}{d}) - x/2 = - \log(n) + o(\log(n)).  
\end{equation} 
We now consider three cases below separately.  
\begin{itemize} 
\item Case 1. $d \ll \log(n)$. 
\item Case 2.  $d = a_0 \log(n)$ for a constant $a_0 > 0$.  
\item Case 3.  $d \gg \log(n)$. 
\end{itemize} 
 
Consider Case 1.  In this case, it is seen that when 
\[
x = O(\log(n)), 
\] 
the LHS of (\ref{equation2}) is
\[
-x/2 + o(\log(n)). 
\] 
Therefore, the solution of (\ref{equation2}) is seen to be 
\[
b_n = (1 + o(1)) \cdot 2 \log(n). 
\] 
 
Consider Case 2. In this case, $d =a_0 \log(n)$.  Let $x = b_1 \log(n)$. 
Plugging these into (\ref{equation2}) and rearranging,  
\begin{equation} \label{equation3} 
a_1 -  a_0 \log(a_1) = 2 + a_0 - a_0 \log(a_0) + o(1). 
\end{equation} 
Now, consider the equation 
\[
a_1 -  a_0 \log(a_1) = 2 + a_0 - a_0 \log(a_0).  
\] 
It is seen that the equation has a unique solution (denoted by $b_0$) that is 
bigger than $2$. 
Therefore,  in this case, 
\[
b_n = (1 + o(1)) b_0, 
\] 

Consider Case 3.  In this case, $d \gg \log(n)$. Consider again the equation 
\[
\log(d/x) + (d/2) \log(\frac{ex}{d}) - x/2 = - \log(n) + o(\log(n)).  
\] 
Letting $y = x/d$ and rearranging, it follows that  
\begin{equation} \label{equation4} 
y - \log(y) - 1 = o(1),  
\end{equation} 
where for sufficiently large $n$, $o(1) > 0$ and $o(1) \goto 0$. 
Note that the function $g(y) = y - \log(y) - 1$ is a convex function with 
a minimum of $0$  reached at  $y = 1$, it follows 
\[
y = 1 + o(1). 
\] 
Recalling $y = x / d$, this shows 
\[
b_n = (1 + o(1)) d. 
\] 
This completes the proof of Lemma \ref{lemma:chi2}.


\section{Analysis of the SPA algorithm}

Fix $d\geq K-1$. For any $V=[v_1,v_2,\ldots,v_K]\in\mathbb{R}^{d\times K}$, let $\sigma_k(V)$ denote the $k$th singular value of $V$, and define 
\[
\gamma(V)=\min_{v_0\in\mathbb{R}^d}\max_{1\leq k\leq K}\|v_k-v_0\|, \qquad d_{\max}(V)=\max_{x\in {\cal S}}\|x\|.
\] 
To capture the error bound for SPA, we introduce a useful quantity in the main paper: 
\beq \label{def:Beta-in-SPA-thm}
\beta(X, V):=\max\biggl\{  \max_{1\leq i\leq n}\mathrm{Dist}(X_i, {\cal S}),\;\;\;  \max_{1\leq k\leq K}\min_{i:r_i=v_k}\|X_i-v_k\| \biggr\}. 
\eeq
We note that when $\max_i\mathrm{Dist}(X_i, {\cal S})$ is small, no point is too far away from the simplex; and when $\max_k \min_{i:r_i=v_k}\|X_i-v_k\|$ is small, there is at least one point near each vertex. 


Let's denote $\gamma=\gamma(V)$, $d_{\max}=d_{\max}(V)$, $\beta=\beta(X,V)$, and $\sigma_*=\sigma_{K-1}(V)$ for brevity.
We shall prove the following theorem, which is a slightly stronger version of Theorem~1 in the main paper.  
\begin{thm} \label{thm:SPA}
Suppose for each $1\leq k\leq K$, there exists $1\leq i\leq n$ such that $\pi_i=e_k$. Suppose $\beta(X,V)$ satisfies that $450d_{\max}\max\bigl\{1, \frac{d_{\max}}{\sigma_*}\bigr\}\beta \leq \sigma^2_*$. 
Let $\hat{v}_1,\hat{v}_2,\ldots,\hat{v}_r$ be the output of SPA. Up to a permutation of these $r$ vectors,
\[
\max_{1\leq k\leq r} \|\hat{v}_k - v_k\|\leq  \Bigl(1+\frac{30\gamma}{\sigma_*}\max\bigl\{1, \frac{d_{\max}}{\sigma_*}\bigr\}\Bigr)\beta(X, V). 
\]
\end{thm}

\subsection{Some preliminary lemmas in linear algebra}

To establish Theorem~\ref{thm:SPA}, it is necessary to develop a few lemmas in linear algebra. First, we notice that the vertex matrix $V$ defines a mapping from the standard probability simplex ${\cal S}^*$ to the target simplex ${\cal S}$. The following lemma gives some properties of the mapping: 
\begin{lem} \label{lem:simplex}
Let ${\cal S}^*\subset\mathbb{R}^K$ be the standard probability simplex consisting of all weight vectors. Let $F:{\cal S}^*\to{\cal S}$ be the mapping with $F(\pi)=V\pi$. For any $\pi$ and $\tilde{\pi}$ in ${\cal S}^*$, 
\beq \label{lem-simplex}
\sigma_{K-1}(V)\cdot \|\pi-\tilde{\pi}\|\leq \|F(\pi)-F(\tilde{\pi})\|\leq \gamma(V)\cdot \|\pi-\tilde{\pi}\|_1. 
\eeq
Fix $1\leq s\leq K-2$. If $\pi$ and $\tilde{\pi}$ share at least $s$ common entries, then  
\beq \label{lem-simplex2}
\|F(\pi)-F(\tilde{\pi})\|\geq \sigma_{K-1-s}(V)\|\pi-\tilde{\pi}\|. 
\eeq
\end{lem}
The first claim of Lemma~\ref{lem:simplex} is about the case where ${\cal S}$ is non-degenerate. In this case, 
\[
\sigma_{K-1}(V)>0.
\]
Hence, we can upper/lower bound the distance between any two points in ${\cal S}$ by the distance between their barycentric coordinates. The second claim considers the case where ${\cal S}$ can be degenerate (i.e., $\sigma_{K-1}(V)=0$ is possible) but 
\[
\sigma_{K-1-s}(V)>0.
\]
We can still use (\ref{lem-simplex}) to upper bound the distance between two points in ${\cal S}$ but the lower bound there is ineffective. Fortunately, if the two points share $s$ common entries in their barycentric coordinates (which implies that the two points are on the same face or edge), then we can still lower bound the distance between them. 

Second, we study the Euclidean norm of a convex combination of $m$ points. Let $w_1,\ldots,w_m$ be the convex combination weights. By the triangle inequality, 
\[
\Bigl\|\sum_{i=1}^mw_ix_i\Bigr\|\leq \sum_{i=1}^m w_i\|x_i\|\leq \max_{1\leq k\leq K}\|v_k\|.
\] 
This explains why $\max_{x\in {\cal S}}\|x\|$ is always attained at a vertex. Write 
\[
\delta:=\sum_{i=1}^m w_i\|x_i\| - \Bigl\|\sum_{i=1}^mw_ix_i\Bigr\|. 
\]
Knowing $\delta\geq 0$ is not enough for showing Theorem~\ref{thm:SPA}. We need to have an explicit lower bound for $\delta$, as given in the following lemma.   

\begin{lem} \label{lem:norm}
Fix $m\geq 2$ and $x_1,\ldots,x_m\in\mathbb{R}^d$. Let $a=\min_{i\neq j}\|x_i-x_j\|$ and $b=\max_{i\neq j}|\|x_i\|-\|x_j\||$. For any $w_1,\ldots,w_m\geq 0$ such that $\sum_{i=1}^mw_i=1$,  
\beq \label{lem-norm}
\Bigl\|\sum_{i=1}^mw_ix_i\Bigr\|\leq L - \frac{a^2-b^2}{4L} \sum_{i=1}^m w_i(1-w_i), \quad\mbox{with}\;\; L:=\sum_{i=1}^m w_i\|x_i\|. 
\eeq 
\end{lem}

By Lemma~\ref{lem:norm}, the lower bound for $\delta$ has the expression $\frac{a^2-b^2}{4L} \sum_{i=1}^m w_i(1-w_i)$. This lower bound is large if $a=\min_{i\neq j}\|x_i-x_j\|$ is properly large, and $b=\max_{i\neq j}|\|x_i\|-\|x_j\||$ is properly small, and $\sum_iw_i(1-w_i)$ is properly large. 
\begin{itemize}
\item A large $a$ means that these $m$ points are sufficiently `different' from each other.
\item A small $b$ means that the norms of these $m$ points are sufficiently close. 
\item A large $\sum_iw_i(1-w_i)$ prevents each of $w_i$ from being too close to $1$, implying that the convex combination is sufficiently `mixed'.  
\end{itemize}
Later in Section~\ref{subsec:simplicial-neighborhoods}, we will see that Lemma~\ref{lem:norm} plays a critical role in the proof of Theorem~\ref{thm:SPA}.   

Third, we explore the projection of ${\cal S}$ into a lower-dimensional space. Let $H\in\mathbb{R}^{d\times d}$ be an arbitrary projection matrix with rank $s$. We use $(I_d-H)$ to project ${\cal S}$ into the orthogonal complement of $H$, where the projected vertices are the columns of 
\[
V^{\perp}=(I_d-H)V.
\] 
Since the projected simplex is not guranteed to be non-degenerate, it is possible that $\sigma_{K-1}(V^{\perp})=0$. However, we have a lower bound for $\sigma_{K-1-s}(V^{\perp})$, as given in the following lemma:  

\begin{lem} \label{lem:cauchy}
Fix $1\leq s\leq K-2$. 
For any projection matrix $H\in\mathbb{R}^{d\times d}$ with rank $s$,  
\beq \label{lem-cauchy}
\sigma_{K-1-s}((I_d-H)V)\geq \sigma_{K-1}(V).
\eeq
\end{lem}

Finally, we present a lemma about 
\[
d_{\max}=\max_{x\in {\cal S}}\|x\|=\max_{1\leq k\leq K}\|v_k\|.
\]
In the analysis of SPA, it is not hard to get a lower bound for $d_{\max}$ in the first iteration. However, as the algorithm successively projects ${\cal S}$ into lower-dimensional subspaces, we need to keep track of this quantity for the projected simplex spanned by $V^{\perp}$. Lemma~\ref{lem:cauchy} shows that the singular values of $V^{\perp}$ can be lower bounded. It motivates us to have a lemma that provides a lower bound of $d_{\max}$ in terms of the singular values of $V$.   
\begin{lem} \label{lem:vertexNorm}
Fix $0\leq s\leq K-2$. Suppose there are at least $s$ indices, $\{k_1,\ldots,k_s\}\subset\{1,2,\ldots,K\}$, such that $\|v_k\|\leq \delta$. If $\sigma^2_{K-1-s}(V)\geq 2(K-2)\delta^2$, then
\beq \label{lem-vertexNorm}
\max_{1\leq k\leq K}\|v_k\|\geq \frac{\sqrt{K-s-1}}{\sqrt{2(K-s)}}\, \sigma_{K-1-s}(V)\geq \frac{1}{2}\sigma_{K-1-s}(V).  
\eeq
\end{lem}

\subsection{The simplicial neighborhoods and a key lemma} \label{subsec:simplicial-neighborhoods}

We fix a simplex ${\cal S}\subset\mathbb{R}^d$ whose vertices are $v_1,v_2,\ldots,v_K$. Write $V=[v_1,v_2,\ldots,v_K]\in\mathbb{R}^{d\times K}$. Let ${\cal S}^*$ denote the standard probability simplex, and let $F:{\cal S}^*\to {\cal S}$ be the mapping in Lemma~\ref{lem:simplex}. We introduce a local neighborhood for each vertex that has a ``simplex shape": 
\begin{definition} \label{def:neighborhood}
Given $\epsilon\in (0,1)$, for each $1\leq k\leq K$, the $\epsilon$-simplicial-neighborhood of $v_k$ inside the simplex ${\cal S}$ is defined by 
\[
{\cal V}_k(\epsilon): = \{F(\pi): \pi\in {\cal S}^*,\, \pi(k) \geq 1-\epsilon \}.
\] 
\end{definition}
These simplicial neighborhoods are highlighted in blue in Figure~\ref{fig:proof}.

\begin{figure}[htb!]
\centering
\includegraphics[width=.5\textwidth]{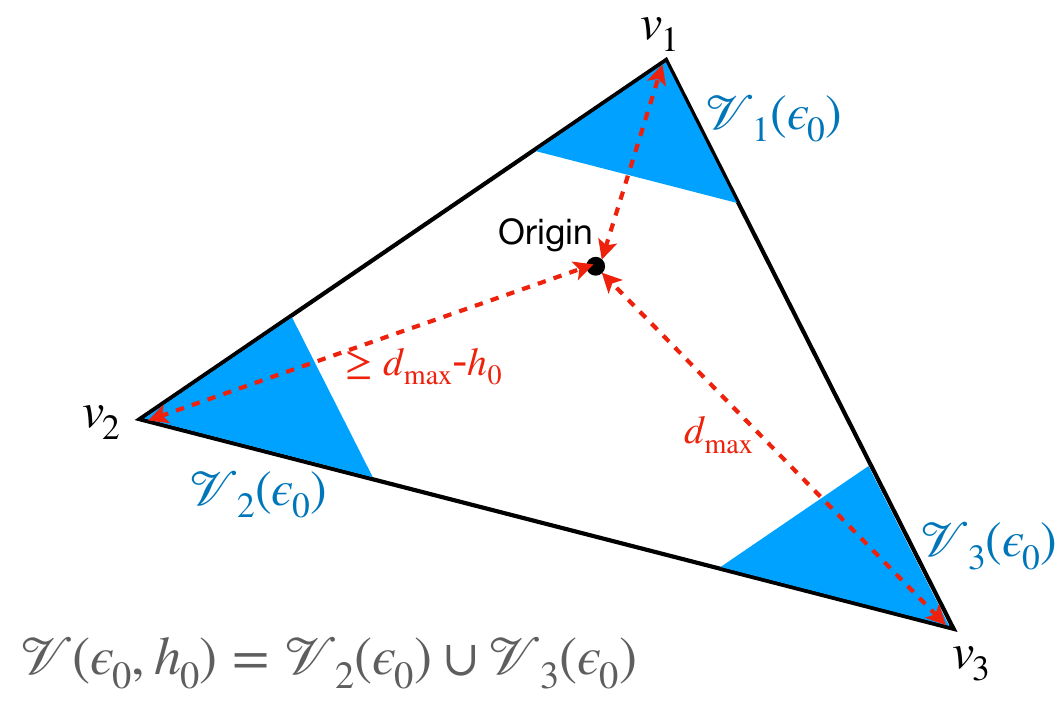}
\caption{An illustration of the simplicial neighborhoods and ${\cal V}(\epsilon_0, h_0)$.} \label{fig:proof}
\end{figure}

First, we verify that each ${\cal V}_k(\epsilon)$ is indeed a ``neighborhood" in the sense each $x\in {\cal V}_k(\epsilon)$ is sufficiently close to $v_k$. Note that $v_k=F(e_k)$, where $e_k$ is the $k$th standard basis vector of $\mathbb{R}^{K}$. For any $\pi\in {\cal S}^*$, 
\[
\|\pi-e_k\|_1=2[1-\pi(k)].
\]
By Definition~\ref{def:neighborhood}, for any $x\in {\cal V}_k(\epsilon)$, its barycentric coordinate $\pi$ satisfies $1-\pi(k)\leq \epsilon$. It follows by Lemma~\ref{lem:simplex} that  
\beq \label{thm-SPA-3}
\max_{x\in {\cal V}_k(\epsilon)}\|x-v_k\| = \max_{\pi\in {\cal S}^*: \pi(k)\leq 1-\epsilon}\|F(\pi)-F(e_k)\| \leq 2\gamma(V) \epsilon. 
\eeq
Hence, ${\cal V}_k(\epsilon)$ is within a ball centered at $v_k$ with a radius of $2\gamma(V)\epsilon$. However, we opt to utilize these simplex-shaped neighborhoods instead of standard balls, as this choice greatly simplifies proofs.

Next, we show that as long as $\epsilon<1/2$, the $K$ neighborhoods ${\cal V}_1(\epsilon), \ldots, {\cal V}_K(\epsilon)$  are non-overlapping. By Lemma~\ref{lem:simplex},  
\beq \label{bwDist}
\|v_k-v_\ell\|\geq \sigma_{K-1}(V)\|e_k-e_\ell\|\geq \sqrt{2}\sigma_{K-1}(V), \qquad\mbox{for }1\leq k\neq \ell\leq K. 
\eeq
When $x\in{\cal V}_k(\epsilon)$, the $k$th entry of $\pi:=F^{-1}(x)$ is at least $1-\epsilon > 1/2$. Since each $\pi\in {\cal S}^*$ cannot have two entries larger than $1/2$, these neighborhoods are disjoint: 
\beq \label{thm-SPA-2(2)}
{\cal V}_k(\epsilon)\cap {\cal V}_{\ell}(\epsilon) = \emptyset, \qquad \mbox{for any }1\leq k\neq \ell\leq K. 
\eeq

{\bf An intuitive explanation of our proof ideas for Theorem~\ref{thm:SPA}}: 
We outline our proof strategy using the example in Figure~\ref{fig:proof}. 
The first step of SPA finds 
\[
i_1=\mathrm{argmax}_{1\leq i\leq n}\|X_i\|.
\]
The population counterpart of $X_{i_1}$ is denoted by $r_{i_1}$. We will explore the region of the simplex that $r_{i_1}$ falls into. In the noiseless case, $X_i=r_i$ for all $1\leq i\leq n$. Since the maximum Euclidean norm over a simplex can only be attained at vertex, $r_{i_1}$ must equal to one of the vertices. In Figure~\ref{fig:proof},  the vertex $v_3$ has the largest Euclidean norm, hence, $r_{i_1}=v_3$ in the noiseless case. In the noisy case, the index $i$ that maximizes $\|X_i\|$ may not maximize $\|r_i\|$; i.e., $r_{i_1}$ may not have the largest Euclidean norm among $r_i$'s. Noticing that $\|v_3\|>\|v_2\|>\|v_1\|$, we expect to see two possible cases:
\begin{itemize}
\item Possibility 1: $r_{i_1}$ is in the $\epsilon$-simplicial-neighborhood of $v_3$, for a small $\epsilon>0$. 
\item Possibility 2 (when $\|v_2\|$ is close to $\|v_3\|$): $r_{i_1}$ is in the $\epsilon$-simplicial-neighborhood of $v_2$. 
\end{itemize}
The focus of our proof will be showing that $r_{i_1}$ falls into ${\cal V}_2(\epsilon)\cup {\cal V}_3(\epsilon)$. No matter $r_i\in {\cal V}_2(\epsilon)$ holds or $r_i\in {\cal V}_3(\epsilon)$ holds, the corresponding $\hat{v}_1=X_{i_1}$ is close to one of the vertices.

{\bf Formalization of the above insights, and a key lemma}: Introduce the notation
\beq \label{thm-SPA-4-add}
{\cal K}^*=\{k: \|v_k\|=d_{\max}\}, \qquad\mbox{where}\quad d_{\max}:=\max_{x\in {\cal S}}\|x\|=\max_{k}\|v_k\|. 
\eeq
Given any $h_0>0$ and $\epsilon_0\in (0,1/2)$, let ${\cal V}_k(\epsilon_0)$ be the same as in Definition~\ref{def:neighborhood},  and we define an index set ${\cal K}(h_0)$ and a region ${\cal V}(\epsilon_0, h_0)\subset{\cal S}$ as follows:
\beq \label{thm-SPA-4}
{\cal K}(h_0) =\{k: \|v_k\|\geq d_{\max}-h_0 \}, \qquad {\cal V}(\epsilon_0, h_0)=\cup_{k\in{\cal K}(h_0)}{\cal V}_k(\epsilon_0), 
\eeq
For the example in Figure~\ref{fig:proof}, ${\cal K}^*=\{3\}$, ${\cal K}(h_0)=\{2,3\}$, and ${\cal V}(\epsilon_0, h_0)={\cal V}_2(\epsilon_0)\cup {\cal V}_3(\epsilon_0)$. 

In the proof of Theorem~\ref{thm:SPA}, we will repeatedly use the following key lemma, which states that the Euclidean norm of any point in ${\cal S}\setminus {\cal V}(\epsilon_0, h_0)$ is strictly smaller than $d_{\max}$ by a certain amount: 
\begin{lem}\label{lem:SPA-supp1}
Fix a simplex ${\cal S}\subset\mathbb{R}^d$ with vertices $v_1,v_2,\ldots,v_K$. Write $d_{\max}=\max_{1\leq k\leq K}\|v_k\|$. Suppose there exists  $\sigma_*>0$ such that  
\beq \label{key-lemma-cond}
d_{\max}\geq \sigma_*/2, \qquad\mbox{and}\qquad  \min_{1\leq k\neq \ell\leq K}\|v_k-v_{\ell}\|\geq \sqrt{2}\sigma_*.  
\eeq
Let ${\cal K}(h_0)$ and ${\cal V}(\epsilon_0, h_0)$ be as defined in (\ref{thm-SPA-4}). 
Given any $t >0$ such that $ \max\{1, d_{\max}/\sigma_*\} t <  3\sigma_*$, if we set $(h_0, \epsilon_0)$ such that 
\beq \label{key-lemma-params}
h_0= \sigma_*/3,  \qquad\mbox{and}\qquad 1/2>\epsilon_0\geq 6\sigma^{-1}_*\max\{1, d_{\max}/\sigma_*\}t ,
\eeq
then 
\beq \label{key-lemma-claim}
\|x\|\leq d_{\max}-t, \qquad\mbox{for all $x\in {\cal S}\setminus {\cal V}(\epsilon_0, h_0)$}.
\eeq
\end{lem}

Lemma~\ref{lem:SPA-supp1} will be proved in Section~\ref{subsec:proof-key-lemma}, where we invoke Lemma~\ref{lem:norm} to prove the claim here.

\subsection{Proof of Theorem~\ref{thm:SPA} (Theorem~1 in the main paper)}

The proof consists of three steps. In Step 1, we study the first iteration of SPA and show that $\hat{v}_1$ falls in the neighborhood of a true vertex. In Steps 2-3, we recursively study the remaining iterations and show that, if $\hat{v}_{1},\ldots,\hat{v}_{s-1}$ fall into the neighborhoods of $(s-1)$ true vertices, one per each, then $\hat{v}_k$ will also fall into the neighborhood of another true vertex. 
For clarity, we first study the second iteration in Step 2 (for which the notations are simpler), and then study the $s$th iteration for a general $s$ in Step 3. 

Let's denote for brevity:
\[
\gamma=\gamma(V), \qquad d_{\max}=d_{\max}(V), \qquad \sigma_*=\sigma_{K-1}(V), \qquad \beta=\beta(X,V). 
\]
Write $J_k=\{1\leq i\leq n: \pi_i(k)=1\}$, for $1\leq k\leq K$. From the definition of $\beta(X,V)$,
\beq \label{thm-SPA-1}
\max_{1\leq i \leq n} \mathrm{Dist}(X_i, {\cal S})\leq \beta, \qquad \max_{1\leq k\leq K}  \min_{i\in J_k}\|X_i - v_k\|\leq \beta. 
\eeq



{\bf Step 1: Analysis of the first iteration of SPA}.

Applying Lemma~\ref{lem:vertexNorm} with $s=0$, we have $d_{\max}\geq \sigma_*/2$. We then apply Lemma~\ref{lem:SPA-supp1}. Let ${\cal V}(\epsilon_0, h_0)$ be as in (\ref{thm-SPA-4}), with 
\beq \label{thm-SPA-new-add-4}
h_0=\sigma_*/3, \qquad\mbox{and}\qquad \epsilon_0=15\max\{\sigma_*, \, \sigma_*^{-2}d_{\max}\}\beta.
\eeq
Our assumptions yield $\epsilon_0<1/2$. Additionally, when $t=7\beta/3$, $\epsilon_0\geq 6\sigma_*^{-1}\max\{1,d_{\max}/\sigma^*\}t$, which satisfies (\ref{key-lemma-params}). 
We apply Lemma~\ref{lem:SPA-supp1} with $t=7\beta/3$. It yields 
\beq \label{thm-SPA-5}
\max_{x\in {\cal S}\setminus {\cal V}(\epsilon_0, h_0)}\|x\|\leq d_{\max}-7\beta/3. 
\eeq
At the same time, let ${\cal K}^*$ be the same as in (\ref{thm-SPA-4-add}). For any $k\in {\cal K}^*$, it follows by (\ref{thm-SPA-1}) that 
\[
\mbox{there exists at least one $i^*\in J_k$ such that $\|X_{i^*}-v_k\|\leq \beta$}.
\]
Note that $\|v_k\|=d_{\max}$ for $k\in {\cal K}^*$. 
It follows by the triangle inequality that 
\[
\|X_{i^*}\|\geq \|v_k\|-\beta \geq  d_{\max}-\beta.
\]
Since $\|X_{i_1}\|=\max_{i}\|X_i\|$, we immediately have: 
\beq \label{thm-SPA-6}
\|X_{i_1}\|\geq \|X_{i^*}\|\geq d_{\max}-\beta. 
\eeq 
Combining (\ref{thm-SPA-5}) and (\ref{thm-SPA-6}), we conclude that $X_{i_1}\notin {\cal S}\setminus {\cal V}(\epsilon_0, h_0)$; in other words,  
\beq
\mbox{$X_{i_1}$ can only be inside ${\cal V}(\epsilon_0, h_0)$ or outside ${\cal S}$}.
\eeq  
Suppose $X_{i_1}$ is outside ${\cal S}$. Let $\mathrm{proj}_{\cal S}(X_{i_1})\in\mathbb{R}^d$ be the point in the simplex that is closest to $X_{i_1}$. In other words,  $\|X_{i_1}-\mathrm{proj}_{\cal S}(X_{i_1})\|=\min_{x\in {\cal S}} \|X_{i_1}-x\|=\mathrm{Dist}(X_{i_1}, {\cal S})$. Using the first inequality in (\ref{thm-SPA-1}), we have
\beq \label{thm-SPA-new-add}
\|X_{i_1}-\mathrm{proj}_{\cal S}(X_{i_1})\|\leq \beta. 
\eeq
It follows by the triangle inequality and (\ref{thm-SPA-6}) that 
\[
\|\mathrm{proj}_{\cal S}(X_{i_1})\| \geq \|X_{i_1}\|-\beta \geq d_{\max}-2\beta. 
\]
Combining it with (\ref{thm-SPA-5}), we conclude that $\mathrm{proj}_{\cal S}(X_{i_1})$ cannot be in ${\cal S}\setminus {\cal V}(\epsilon_0, h_0)$. So far, we have shown that one of the following cases must happen: 
\begin{align} \label{thm-SPA-cases}
& \mbox{Case 1: $X_{i_1}\in {\cal V}(\epsilon_0, h_0)$}, \cr
&  \mbox{Case 2: $X_{i_1}\notin {\cal S}$, and $\mathrm{proj}_{\cal S}(X_{i_1})\in {\cal V}(\epsilon_0, h_0)$}. 
\end{align}

In Case 1, since ${\cal V}_1(\epsilon_0), \ldots,{\cal V}_K(\epsilon_0)$ are disjoint, 
there exists only one $k_1\in {\cal K}(h_0)$ such that $X_{i_1}\in {\cal V}_{k_1}(\epsilon_0)$. It follows by (\ref{thm-SPA-3}) that 
\beq \label{thm-SPA-new-add-2}
\|X_{i_1}-v_{k_1}\|\leq 2\gamma\epsilon_0, \qquad \mbox{in Case 1}. 
\eeq
In Case 2, similarly, there is only one $k_1\in {\cal K}(h_0)$ such that $\mathrm{proj}_{\cal S}(X_{i_1})\in {\cal V}_{k_1}(\epsilon_0)$. It follows by (\ref{thm-SPA-3}) again that 
\[
\|\mathrm{proj}_{\cal S}(X_{i_1})-v_{k_1}\|\leq 2\gamma\epsilon_0.
\]
Combining it with (\ref{thm-SPA-new-add}) gives
\begin{align} \label{thm-SPA-new-add-3}
\|X_{i_1}-v_{k_1}\| & \leq \|X_{i_1}-\mathrm{proj}_{\cal S}(X_{i_1})\| + \|\mathrm{proj}_{\cal S}(X_{i_1}) -v_{k_1} \| \cr
& \leq 2\gamma\epsilon_0 + \beta, \qquad\mbox{in Case 2}. 
\end{align}
We put (\ref{thm-SPA-new-add-2}) and (\ref{thm-SPA-new-add-3}) together and plug in the value of $\epsilon_0$ in (\ref{thm-SPA-new-add-4}). It yields:
\begin{align} \label{thm-SPA-claim1}
\|X_{i_1}&-v_{k_1}\|  \leq \beta + 2\gamma\epsilon_0\cr
&\leq  \Bigl(1+\frac{30\gamma}{\sigma_*}\max\bigl\{1, \frac{d_{\max}}{\sigma_*}\bigr\}\Bigr)\beta, \qquad\mbox{for some $k_1$}.  
\end{align}

{\bf Step 2: Analysis of the second iteration of SPA}. 

Let $H_1=I_d - \frac{1}{\|X_{i_1}\|^2}X_{i_1}X'_{i_1}$ and $\widetilde{X}_i =H_1X_i$, for $1\leq i\leq n$. The second iteration operates on the data points $\widetilde{X}_1,\ldots, \widetilde{X}_n\in\mathbb{R}^d$. Write 
\[
\tilde{r}_i = H_1r_i, \qquad \tilde{\epsilon}_i=H_1\epsilon_i, \qquad \tilde{v}_k = H_1v_k, \qquad \widetilde{V}=[\tilde{v}_1, \tilde{v}_2,\ldots,\tilde{v}_K].
\]
It follows that
\beq \label{thm-SPA-newmodel}
\widetilde{X}_i = \widetilde{V}\pi_i + \tilde{\epsilon}_i, \qquad 1\leq i\leq n. 
\eeq
Let $\widetilde{S}\subset\mathbb{R}^d$ denote the projected simplex, whose vertices are $\tilde{v}_1,\ldots,\tilde{v}_K$. Let $\widetilde{F}$ denote the mapping from the standard probability simplex ${\cal S}^*$ to the projected simplex $\widetilde{S}$ (note that $\widetilde{F}$ is not necessarily  a one-to-one mapping). We consider the neighborhoods of $\widetilde{S}$ using Definition~\ref{def:neighborhood}
\beq \label{thm-SPA-newNeighbor}
\widetilde{\cal V}_k(\epsilon) = \bigl\{\widetilde{F}(\pi): \pi\in {\cal S}^*, \, \pi_i(k) \geq 1-\epsilon \bigr\} \subset \mathbb{R}^d, \qquad 1\leq k\leq K. 
\eeq
Let $k_1$ be as in (\ref{thm-SPA-claim1}). Let $\tilde{d}_{\max}:=\max_{x\in \widetilde{S}}\|x\|$. 
The maximum distance $\tilde{d}_{\max}$ is attained at one or multiple vertices. Same as before, let $\widetilde{\cal K}^*$ be the 
index set of $k$ at which $\|\tilde{v}_k\|=\tilde{d}_{\max}$. We similarly define
\beq \label{thm-SPA-induction-1}
\widetilde{\cal K}(h_0) =\{k: \|\tilde{v}_k\|\geq \tilde{d}_{\max}-h_0 \}, \qquad \widetilde{\cal V}(\epsilon_0, h_0)=\cup_{k\in \widetilde{\cal K}(h_0)}\widetilde{\cal V}_k(\epsilon_0). 
\eeq
At the same time, let $\tilde{\beta}=\beta( \widetilde{X}, \widetilde{V})$. It is easy to see that for any points $x$ and $y$, $\|H_1x-H_1y\|\leq \|x-y\|$. Hence, $\tilde{\beta}\leq \beta$. It follows that
\beq \label{thm-SPA-induction-2}
\max_{1\leq i \leq n} \mathrm{Dist}(\widetilde{X}_i, \widetilde{\cal S})\leq \beta, \qquad \max_{1\leq k\leq K}  \min_{i\in J_k}\|\widetilde{X}_i - \tilde{v}_k\|\leq \beta. 
\eeq
Additionally, we have the following lemma:
\begin{lem} \label{lem:SPA-tech-lem1}
Under the conditions of Theorem~\ref{thm:SPA}, for $\sigma_*=\sigma_{K-1}(V)$, the following claims are true: 
\beq \label{thm-SPA-induction-3}
\tilde{d}_{\max}\geq \sigma_*/2, \quad \min_{\substack{(k,\ell): k\neq k_1,\\\ell\neq k_1, k\neq \ell}}\|\tilde{v}_k-\tilde{v}_{\ell}\|\geq \sqrt{2}\sigma_*, \quad\mbox{and}\quad k_1 \notin \widetilde{\cal K}(h_0). 
\eeq
\end{lem}
Given (\ref{thm-SPA-newmodel})-(\ref{thm-SPA-induction-3}), we now apply Lemma~\ref{lem:SPA-supp1} to study the projected simplex $\widetilde{S}$. 
Similarly as how we obtain (\ref{thm-SPA-5}), by choosing 
\[
h_0=\sigma_*/3,  \qquad\mbox{and}\qquad \epsilon_1=15\max\{\sigma_*, \, \sigma_*^{-2}\tilde{d}_{\max}\},
\]
we get $
\max_{x\in \widetilde{\cal S}\setminus \widetilde{\cal V}(\epsilon_1, h_0)}\|x\|\leq \tilde{d}_{\max}-7\beta/3$. 
Note that $\epsilon_1\leq \epsilon_0$, and the set $\widetilde{S}\setminus\widetilde{V}(\epsilon, h_0)$ becomes smaller as $\epsilon$ increases. We immediately have  
\beq \label{thm-SPA-induction-4}
 \max_{x\in \widetilde{\cal S}\setminus \widetilde{\cal V}(\epsilon_0, h_0)}\|x\| \leq \tilde{d}_{\max}-7\beta/3. 
\eeq 
At the same time, by (\ref{thm-SPA-induction-2}) and (\ref{thm-SPA-induction-3}), it is easy to get (similar to how we obtained (\ref{thm-SPA-6}))
\[
\|\tilde{X}_{i_2}\|\geq \tilde{d}_{\max}-\beta.
\]
We can mimic the analysis between (\ref{thm-SPA-6}) and (\ref{thm-SPA-cases}) to show that one of the two cases happens: 
\begin{align} \label{thm-SPA-cases-proj}
& \mbox{Case 1: $\widetilde{X}_{i_2}\in \widetilde{\cal V}(\epsilon_0, h_0)$}, \cr
&  \mbox{Case 2: $\widetilde{X}_{i_2}\notin \widetilde{\cal S}$, and $\mathrm{proj}_{\widetilde{\cal S}}(\widetilde{X}_{i_2})\in \widetilde{\cal V}(\epsilon_0, h_0)$}. 
\end{align}
Consider Case~1. Since $H_1$ is a linear projector, $\widetilde{X}_i\in \widetilde{\cal V}_k(\epsilon_0)$ if and only if $X_i\in {\cal V}_k(\epsilon_0)$. 
Hence,   
\[
X_{i_2}\in \bigl(\cup_{k\in \widetilde{\cal K}(h_0)}{\cal V}_k(\epsilon_0)\bigr). 
\]
There exists a unique $k_2\in \widetilde{\cal K}(h_0)$ such that $X_{i_2}\in {\cal V}_{k_2}(\epsilon_0)$. It follows by (\ref{thm-SPA-3}) that 
\[
\|X_{i_2}-v_{k_2}\|\leq 2\gamma \epsilon_0, \qquad\mbox{in Case 1}. 
\]
Consider Case~2. Write $\tilde{x}=\mathrm{proj}_{\widetilde{\cal S}}(\widetilde{X}_{i_2})$ for short, and let $M=\{x\in {\cal S}: H_1x=\tilde{x}\}$. For any $k$, $\tilde{x}\in \widetilde{\cal V}_k(\epsilon_0)$ implies that $x\in {\cal V}_k(\epsilon_0)$ for every $x\in M$. Additionally, $\widetilde{X}_i\in \widetilde{\cal S}$ if and only if $X_i\in {\cal S}$. Hence, it holds in Case~2 that 
\[
X_{i_2}\notin {\cal S}, \mbox{ and } x\in \bigl(\cup_{k\in \widetilde{\cal K}(h_0)}{\cal V}_k(\epsilon_0)\bigr), \mbox{ for every }x\in M. 
\]
We pick one $x\in M$. There exists a unique $k_2\in \widetilde{\cal K}(h_0)$ such that $x\in {\cal V}_{k_2}(\epsilon_0)$. By mimicking the derivation of (\ref{thm-SPA-new-add-3}), we obtain that
\[
\|X_{i_2}-v_{k_2}\|\leq 2\gamma \epsilon_0 + \beta, \qquad\mbox{in Case 2}. 
\]
Combining the two cases and using the value of $\epsilon_0$ in (\ref{thm-SPA-new-add-4}), we have the conclusion as
\beq \label{thm-SPA-claim2}
\|X_{i_2}-v_{k_2}\|\leq\Bigl(1+\frac{30\gamma}{\sigma_*}\max\bigl\{1, \frac{d_{\max}}{\sigma_*}\bigr\}\Bigr)\beta, \qquad\mbox{for some $k_2\neq k_1$}. 
\eeq 

{\bf Step 3: Analysis of the remaining iterations of SPA}. 

Fix $3\leq s\leq K-1$. We now study the $s$th iteration. Let $i_1,\ldots,i_K$ denote the sequentially selected indices in SPA. 
We aim to show that there exist distinct $k_1, k_2,\ldots,k_s\in\{1,2,\ldots,K\}$ such that 
\beq \label{thm-SPA-claim3}
\|X_{i_s}-v_{k_s}\|\leq\Bigl(1+\frac{30\gamma}{\sigma_*}\max\bigl\{1, \frac{d_{\max}}{\sigma_*}\bigr\}\Bigr)\beta.  
\eeq 
Let's denote ${\cal M}_{s-1}:=\{k_1,\ldots,k_{s-1}\}$ for brevity. Suppose we have already shown (\ref{thm-SPA-claim3}) for every index $1,2,\ldots, s-1$. Our goal is showing that (\ref{thm-SPA-claim3}) continues to hold for $s$ and some $k_s\notin {\cal M}_{s-1}$. 

Let $X_i^{(1)}=X_i$ and $H_1$ be the same as in Step 1 of this proof. We define $X_i^{(s)}$ and $H_s$ recursively to describe the iterations in SPA:  
\beq \label{induction-s-notations}
\hat{y}_{s-1}=\frac{X_{i_{s-1}}^{(s-1)}}{\|X_{i_{s-1}}^{(s-1)}\|}, \qquad H_s=(I_d - \hat{y}_{s-1}\hat{y}_{s-1})H_{s-1},\qquad X_i^{(s)}=H_sX_i^{(s-1)}. 
\eeq
It is seen that $H_{s-1}=\prod_{m=1}^{s-1}(I_d-\hat{y}_m\hat{y}_m')$. Note that each $\hat{y}_m$ is orthogonal to $\hat{y}_1,\ldots,\hat{y}_{m-1}$. As a result, $H_{s-1}$ is a projection matrix with rank $(s-1)$. We apply Lemma~\ref{lem:cauchy} to obtain that
\beq \label{thm-SPA-iter-1}
\sigma_{K-s}(H_{s-1}V)\geq \sigma_{K-1}(V)\geq \sigma_*, \qquad \mbox{for }3\leq s\leq K-1.  
\eeq

Write $V^{(s-1)}=H_{s-1}V$ and $V^{(s)}=H_{s}V$. 
Using the notations in (\ref{induction-s-notations}), we have 
\[
X^{(s)}_i= (I_d-\hat{y}_s\hat{y}_s')X^{(s-1)}_i,  \qquad V^{(s)}=(I_d-\hat{y}_s\hat{y}_s')V^{(s-1)}.
\]
Here, $\Gamma_s:=I_d-\hat{y}_s\hat{y}_s'$ is a projection matrix. 
We observe: 
\beq \label{induction-s-key}
\begin{array}{l}
\mbox{The relationship between $(X^{(s-1)}_i, V^{(s-1)})$ and $(X^{(s)}_i, V^{(s)})$ is similar to the one}\\
\mbox{between $(X_i, V)$ and $(\widetilde{X}_i, \widetilde{V})$ in Step~2, except that $H_1$ is replaced with $\Gamma_s$.}
\end{array}
\eeq  

We aim to show that (\ref{thm-SPA-newmodel})-(\ref{thm-SPA-induction-2}) still hold when those quantities are defined through $(X^{(s)}_i, V^{(s)})$. 
Recall that the proofs in Step~2 are inductive, where we actually showed that if (\ref{thm-SPA-newmodel})-(\ref{thm-SPA-induction-2}) hold for the corresponding quantities defined through $(X_i,V)$, then they also hold for the same quantities defined through $(\widetilde{X}_i,\widetilde{V})$. Given (\ref{induction-s-key}), the same is true here.


It remains to develop a counterpart of Lemma~\ref{lem:SPA-tech-lem1}. The following lemma will be in Section~\ref{subsec:proof-techlemma-2}. It is also an inductive proof, relying on that (\ref{thm-SPA-claim3}) already holds for $1,2,\ldots, s-1$. 
. 
\begin{lemma} \label{lem:SPA-tech-lem2}
Under the conditions of Theorem~\ref{thm:SPA}, write $\sigma_*=\sigma_{K-1}(V)$. Let $\tilde{v}_k=V^{(s)}e_k$, $\tilde{d}_{\max}=\max_{k}\|\tilde{v}_k\|$, and  $\widetilde{\cal K}(h_0) =\{k: \|\tilde{v}_k\|\geq \tilde{d}_{\max}-h_0 \}$. The following claims are true: 
\beq \label{thm-SPA-iter-2}
\tilde{d}_{\max}\geq \sigma_*/2, \quad \min_{\substack{\{k,\ell\}\cap {\cal M}_{s-1}=\emptyset,\\ k\neq \ell}}\|\tilde{v}_k-\tilde{v}_{\ell}\|\geq \sqrt{2}\sigma_*, \quad\mbox{and}\quad {\cal M}_{s-1} \cap \widetilde{\cal K}(h_0)=\emptyset. 
\eeq
\end{lemma}
In Step~2, we have carefully shown how to use (\ref{thm-SPA-newmodel})-(\ref{thm-SPA-induction-3}) to get (\ref{thm-SPA-claim2}). Using similar analyses, we can use the counterparts of (\ref{thm-SPA-newmodel})-(\ref{thm-SPA-induction-2}), which are defined through $(X^{(s)}_i, V^{(s)})$, and the claim of Lemma~\ref{lem:SPA-tech-lem2}, to obtain (\ref{thm-SPA-claim3}). This completes the proof.

\subsection{Proof of the supplementary lemmas}

\subsubsection{Proof of Lemma~\ref{lem:simplex}}
By definition, $F(\pi)=\sum_{k=1}^K\pi(k)v_k$. Since $\sum_{k=1}^K\pi(k)=1$, for any $v_0\in\mathbb{R}^d$, we can re-express $F(\pi)$ as $F(\pi)=v_0+\sum_{k=1}^K\pi(k)(v_k-v_0)$. It follows immediately that
\[
\|F(\pi)-F(\tilde{\pi})\|=\biggl\| \sum_{k=1}^K[\pi(k)-\tilde{\pi}(k)](v_k-v_0) \biggr\|\leq \|\pi-\tilde{\pi}\|_1\cdot\max_{k}\|v_k-v_0\|. 
\]
At the same time, since ${\bf 1}_K'(\pi-\tilde{\pi})=0$, the vector $\pi-\tilde{\pi}$ is an $(K-1)$-dimensional linear subspace. It follows by basic properties of singular values that 
\[
\|F(\pi)-F(\tilde{\pi})\|=\|V(\pi-\tilde{\pi})\|\geq \sigma_{K-1}(V)\cdot\|\pi-\tilde{\pi}\|. 
\]
Combining the above gives (\ref{lem-simplex}). 

Suppose there are $1\leq k_1<k_2<\ldots <k_s\leq K$ such that $\pi(k_j)=\tilde{\pi}(k_j)$, for $1\leq j\leq s$. Then, the vector $\delta = \pi-\tilde{\pi}$ satisfies $(s+1)$ constraints: ${\bf 1}_K'\delta=0$, $\delta(k_j)=0$, for $1\leq j\leq s$. In other words, $\delta$ lives in a $(K-1-s)$-dimensional linear space. It follows by properties of singular values that 
\[
\|F(\pi)-F(\tilde{\pi})\|=\|V(\pi-\tilde{\pi})\|\geq \sigma_{K-1-s}(V)\cdot\|\pi-\tilde{\pi}\|. 
\]
This proves (\ref{lem-simplex2}).

\subsubsection{Proof of Lemma~\ref{lem:norm}}
Write for short $x=\sum_{i=1}^m\pi_ix_i\in\mathbb{R}^d$ and $L=\sum_{i=1}^m w_i\|x_i\|$. By the triangle inequality,
\[
\|x\|\leq L.  
\]
In this lemma, we would like to get a lower bound for $L-\|x\|$. By definition,
\beq \label{lem-norm-1}
 \|x\|^2 =\sum_i w_i^2 \|x_i\|^2 + \sum_{ i\neq j}w_iw_j x_i'x_j.  
\eeq
For any vectors $u,v\in\mathbb{R}^d$, we have a universal equality: $
2u'v = 2\|u\|\|v\|+ (\|u\|-\|v\|)^2-\|u-v\|^2$.  
By our assumption, $\|x_i-x_j\|\geq a$ and $(\|x_i\|-\|x_j\|)^2\leq b^2$, for all $i\neq j$. 
It follows that
\beq \label{lem-norm-0}
x_i'x_j\leq \|x_i\|\|x_j\|-(a^2-b^2)/2, \qquad 1\leq i\neq j\leq m. 
\eeq
We plug (\ref{lem-norm-0}) into (\ref{lem-norm-1}) to get
\begin{align} \label{lem-norm-2}
 \|x\|^2 
&\leq \sum_i w_i^2 \|x_i\|^2 +\sum_{i\neq j} w_iw_j \|x_i\|\|x_j\|- \frac{1}{2}(a^2-b^2)\sum_{i\neq j}w_iw_j  \cr
&= L^2- \frac{1}{2}(a^2-b^2)\sum_{i\neq j}w_iw_j. 
\end{align}
Note that $\sum_{i\neq j}w_iw_j =\sum_i\sum_{j:i\neq j}w_j = \sum_i w_i(1-w_i)$. Combining it with 
(\ref{lem-norm-2}) gives
\beq \label{lem-norm-4}
 \|x\|^2 \leq L^2 - \frac{1}{2} (a^2-b^2)\sum_i w_i(1-w_i). 
\eeq
At the same time, $L+\|x\|\leq 2L$. It follows that
\beq \label{lem-norm-5}
L-\|x\| = \frac{L^2-\|x\|^2}{L+\|x\|}\geq \frac{L^2-\|x\|^2}{2L}\geq \frac{a^2-b^2}{4L}\sum_i w_i(1-w_i).  
\eeq
This proves the claim.

\subsubsection{Proof of Lemma~\ref{lem:cauchy}}
Since $H$ is a projection matrix, there exists $Q_1\in\mathbb{R}^s$ and $Q_2\in\mathbb{R}^{d-s}$ such that $Q=[Q_1,Q_2]$ is an orthogonal matrix, $H=Q_1Q_1'$, and $I_d-H=Q_2Q_2'$. It follows that
\[
(I_d-H)VV'(I_d-H) = Q_2(Q_2'VV'Q_2)Q_2'.  
\]
Since $Q_2$ has orthonormal columns, for any symmetric matrix $M\in\mathbb{R}^{(d-s)\times (d-s)}$, $M$ and $Q_2MQ_2'$ have the same set of nonzero eigenvalues. Hence,
\[
\sigma_{K-1-s}^2((I_d-H)V) = \lambda_{K-1-s}(Q_2'VV'Q_2). 
\]
We note that $Q_2'VV'Q_2\in\mathbb{R}^{(d-s)\times (d-s)}$ is a principal submatrix of $Q'VV'Q\in\mathbb{R}^{d\times d}$. Using the eigenvalue interlacing theorem \citep[Theorem 4.3.28]{HornJohnson},
\[
\lambda_{K-1-s}(Q_2'VV'Q_2)\geq \lambda_{K-1}(Q'VV'Q). 
\]
The claim follows immediately by noting that $\lambda_{K-1}(Q'VV'Q)= \lambda_{K-1}(VV')=\sigma^2_{K-1}(V)$.

\subsubsection{Proof of Lemma~\ref{lem:vertexNorm}}
Write $\ell_{\max}=\max_{1\leq k\leq K}\|v_k\|$. We target to show
\beq \label{lem-vertexNorm-0}
\ell^2_{\max}\geq \frac{K-s-1}{2(K-s)}\sigma_*^2, \qquad\mbox{with } \sigma_*:= \sigma_{K-1-s}(V). 
\eeq
The right hand side of (\ref{lem-vertexNorm-0}) is minimized at $s=K-2$, at which $\ell^2_{\max}\geq \sigma^2_*/4$. We now show (\ref{lem-vertexNorm-0}). 
When $s=0$, it is seen that 
\[
K\ell_{\max}^2\geq \sum_k\|v_k\|^2=\mathrm{trace}(V'V)\geq (K-1)\sigma^2_{K-1}(V). 
\]
Therefore, $\ell_{\max}^2\geq \frac{K-1}{K}\sigma^2_*$, which implies (\ref{lem-vertexNorm}) for $s=0$. When $1\leq s\leq K-2$, since $\|v_k\|\leq \delta$ for at least $s$ of the vertices,   
\[
s\delta^2+(K-s)\ell^2_{\max}\geq \sum_k\|v_k\|^2= \mathrm{trace}(V'V)\geq (K-1-s)\sigma^2_{K-1-s}(V). 
\]
As a result, for $\sigma_*=\sigma_{K-1-s}(V)$, 
\beq \label{lem-vertexNorm-1}
\ell^2_{\max}\geq \frac{(K-s-1)\sigma^2_*-s\delta^2}{K-s}. 
\eeq
Note that $\frac{s}{K-s-1}$ is a monotone increasing function of $s$. Hence, $\frac{s}{K-s-1}\leq K-2$.  The assumption of $2(K-2)\delta^2\leq \sigma^2_*$ implies that $\frac{2s}{K-s-1}\delta^2\leq \sigma^2_*$, or equivalently, $s\delta^2\leq \frac{K-s-1}{2}\sigma^2_*$. We plug it into (\ref{lem-vertexNorm-1}) to get $\ell^2_{\max}\geq \frac{K-s-1}{2(K-s)}\sigma^2_*$.
This proves  (\ref{lem-vertexNorm}) for $1\leq s\leq K-2$.

\subsubsection{Proof of Lemma~\ref{lem:SPA-supp1}} \label{subsec:proof-key-lemma}

Write ${\cal K}={\cal K}(h_0)$, ${\cal V}_k={\cal V}_k(\epsilon_0)$, and ${\cal V}={\cal V}(\epsilon_0, h_0)$ for short. 
By definition of ${\cal K}$, 
\beq \label{SPAsupp-00}
d_{\max}-h_0\leq \|v_k\|\leq d_{\max}, \mbox{ for $k\in {\cal K}$},\quad
\|v_k\|\leq d_{\max}-h_0, \mbox{ for $k\notin{\cal K}$}. 
\eeq
We shall fix a point $x\in {\cal S}\setminus {\cal V}$ and derive an upper bound for $\|x\|$.

First, we need some preparation, let $F$ be the mapping in Lemma~\ref{lem:simplex}. It follows that $\pi=F^{-1}(x)$ is the barycentric coordinate of $x$ in the simplex. By definition of ${\cal V}$, 
\beq \label{SPAsupp-0}
\max_{k\in {\cal K}} \pi(k)\leq 1-\epsilon_0, \qquad \mbox{whenever $x:=F(\pi)$ is in }{\cal S}\setminus {\cal V}. 
\eeq
The $K$ vertices are naturally divided into two groups: those in ${\cal K}$ and those not in ${\cal K}$. Define
\beq \label{SPAsupp-rho-eta}
\rho:=\sum_{k\in {\cal K}}\pi(k), \qquad \eta:=\begin{cases} \rho^{-1}\sum_{k\in {\cal K}}\pi(k)v_k, &\mbox{if }\rho\neq 0,\\ {\bf 0}_d, & \mbox{otherwise}. \end{cases} 
\eeq
Here, $\rho$ is the total weight $\pi$ puts on those vertices in ${\cal K}$, and we can re-write $x$ as
\[
x = \rho \eta + \sum_{k\notin {\cal K}}\pi(k)v_k. 
\]
By the triangle inequality, 
\begin{align} \label{SPAsupp-2}
\|x\| &=\Bigl\| \rho \eta + \sum_{k\notin {\cal K}}\pi(k)v_k\Bigr\| \leq \rho\|\eta\| +\sum_{k\notin{\cal K}}\pi(k)\|v_k\|\cr
&\leq \rho\|\eta\|+(1-\rho)(d_{\max}-h_0).  
\end{align}

Next, we proceed with showing the claim. We consider two cases:
\[
1-\rho \geq \epsilon_0/2\; \mbox{ (Case 1)}, \qquad \mbox{and}\qquad 1-\rho < \epsilon_0/2\; \mbox{ (Case 2)}. 
\] 
In Case~1, the total weight that $\pi_i$ puts on those vertices not in ${\cal K}$ is at least $\epsilon_0/2$. Since each vertex satisfies that $\|v_k\|\leq d_{\max}-h_0$ (see (\ref{SPAsupp-0})) and $\|\eta\|\leq d_{\max}$, 
it follows from (\ref{SPAsupp-2}) that
\beq \label{SPAsupp-case1}
\|x\|\leq d_{\max}-(1-\rho) h_0\leq d_{\max}-\frac{h_0\epsilon_0}{2}, \qquad\mbox{in Case 1}. 
\eeq
In Case~2, if ${\cal K}=\{k^*\}$ is a singleton, then $\rho=\pi(k^*)$. By (\ref{SPAsupp-0}), $\pi(k^*)\leq 1-\epsilon_0$, which leads to $1-\rho=1-\pi(k^*)\geq \epsilon_0$. 
This yields a contradiction to $1-\rho < \epsilon_0/2$. Hence, it must hold that 
\beq \label{SPAsupp-add}
|{\cal K}|\geq 2.
\eeq
Now, $\eta$ is a convex combination of more than one point in $\{v_k: k\in {\cal K}\}$, for which we hope to apply Lemma~\ref{lem:norm}. By (\ref{SPAsupp-00}), for each $k\in {\cal K}$, $\|v_k\|$ is in the interval $[d_{\max}-h_0, d_{\max}]$. Hence, we can take $b=h_0$ in Lemma~\ref{lem:norm}. In addition, from the assumption (\ref{key-lemma-cond}), $\|v_k-v_\ell\|\geq \sqrt{2}\sigma_*$ for any $k\neq \ell$. 
Hence, we set $a=\sqrt{2}\sigma_*$ in Lemma~\ref{lem:norm}. 
We apply this lemma to the vector $\eta$ in (\ref{SPAsupp-rho-eta}). It yields  
\beq \label{SPAsupp-add-2}
\|\eta\|\leq L -\frac{(2\sigma^2_*-h_0^2)}{4L}\sum_{k\in {\cal K}}\frac{\pi(k)[\rho-\pi(k)]}{\rho^2}, \qquad\mbox{with}\quad L:=\sum_{k\in {\cal K}}\frac{\pi(k)}{\rho}\|v_k\|. 
\eeq 
Since $L\leq d_{\max}$, it follows from (\ref{SPAsupp-add-2}) that  
\[
\|\eta\|\leq d_{\max} -\frac{2\sigma^2_*-h_0^2}{4\rho d_{\max}}\sum_{k\in {\cal K}}\pi(k)[1-\rho^{-1}\pi(k)].
\]
Additionally, noticing that $\pi(k)\leq 1-\epsilon_0$ for each $k\in {\cal K}$, we have the following inequality: 
\[
1-\rho^{-1}\pi(k)=\rho^{-1}[1-\pi(k)]-\rho^{-1}(1-\rho)\geq \rho^{-1}[\epsilon_0-(1-\rho)]. 
\]
Combining these arguments and using the fact that $\sum_{k\in {\cal K}}\pi(k)=\rho$, we have 
\begin{align} \label{SPAsupp-add-3}
\|\eta\| & \leq d_{\max}-\frac{(2\sigma^2_*-h_0^2)[\epsilon_0-(1-\rho)]}{4\rho^2 d_{\max}}\sum_{k\in {\cal K}}\pi(k)\cr
&\leq d_{\max}-\frac{(2\sigma^2_*-h_0^2)[\epsilon_0-(1-\rho)]}{4\rho d_{\max}}. 
\end{align}
Since $1-\rho\leq \epsilon_0/2$, we immediately have $\|\eta\|\leq d_{\max}-\frac{2\sigma^2_*-h_0^2}{8\rho d_{\max}}$.  We plug it into (\ref{SPAsupp-2}) to get 
\begin{align} \label{SPAsupp-case3}
\|x\| &\leq \rho\Bigl( d_{\max}-\frac{2\sigma^2_*-h_0^2}{8\rho d_{\max}}\Bigr) + (1-\rho)(d_{\max}-h_0) \cr
&\leq \rho\Bigl( d_{\max}-\frac{2\sigma^2_*-h_0^2}{8\rho d_{\max}}\Bigr) + (1-\rho)d_{\max}\cr
&\leq d_{\max}-\frac{(2\sigma^2_*-h_0^2)\epsilon_0}{8 d_{\max}}, \qquad\mbox{in Case 2}. 
\end{align}

We now combine (\ref{SPAsupp-case1}) for Case~1 and (\ref{SPAsupp-case3}) for Case~2. By setting $h_0=\sigma_*/3$, we have a unified expression:
\[
\|x\|\leq d_{\max}-\min\Bigl\{\frac{\sigma_*}{6},\; \frac{2\sigma_*^2}{9d_{\max}} \Bigr\}\epsilon_0.   
\]
Consequently, a sufficient condition for $\|x\|\leq d_{\max}-t$ to hold is 
\[
\min\Bigl\{\frac{\sigma_*}{6},\; \frac{\sigma_*^2}{6d_{\max}} \Bigr\}\epsilon_0\leq t \qquad \Longleftrightarrow \qquad \epsilon_0\geq \frac{6}{\sigma^*}\max\Bigl\{1, \, \frac{d_{\max}}{\sigma_*} \Bigr\} t.
\]
This proves the claim.

\subsubsection{Proof of Lemma~\ref{lem:SPA-tech-lem1}}
Without loss of generality, we assume $k_1=1$. 

By definition, $\widetilde{V}=H_1V$, where $H_1$ is a rank-$1$ projection matrix. It follows by Lemma~\ref{lem:cauchy} that 
\beq \label{tech-lem1-0}
\sigma_{K-2}(\widetilde{V})\geq \sigma_{K-1}(V)= \sigma_*.
\eeq
Note that $\tilde{d}_{\max}\geq \max_{k\neq 1}\|\tilde{v}_k\|$ and $\|\tilde{v}_1\|=0$. We apply Lemma~\ref{lem:vertexNorm} with $s=1$ and $\delta=0$ to get 
\[
\tilde{d}_{\max}\geq \frac{1}{2}\sigma_{K-2}(\widetilde{V}) \geq \frac{1}{2}\sigma_*.
\]
This proves the first claim in (\ref{thm-SPA-induction-3}). Note that $\tilde{v}_k=\widetilde{V}e_k$, where $e_k\in\mathbb{R}^K$ is a standard basis vector. For any $2\leq k\neq \ell \leq K$, $e_k$ and $e_\ell$ both have a zero at the first coordinate; and we apply Lemma~\ref{lem:simplex} with $s=1$ to get 
\[
\|v_k-v_\ell\|\geq \sigma_{K-2}(\widetilde{V})\|e_k-e_{\ell}\|\geq \sqrt{2}\sigma_*.
\] 
This proves the second claim in (\ref{thm-SPA-induction-3}). 

Finally, we show the third claim. Note that 
\beq \label{tech-lem1-1}
\tilde{v}_1=H_1v_1 = v_{1}-\frac{v_1'X_{i_1}}{\|X_{i_1}\|^2}X_{i_1}=\frac{X_{i_1}'(X_{i_1}-v_1)}{\|X_{i_1}\|^2}v_1-\frac{v_1'X_{i_1}}{\|X_{i_1}\|^2}(X_{i_1}-v_1).
\eeq
Here, $\|v_1\|\leq d_{\max}$, and by (\ref{thm-SPA-6}), $\|X_{i_1}\|\geq d_{\max}-\beta$. Since $|X_{i_1}'(X_{i_1}-v_1)|\leq \|X_{i_1}\|\cdot\|X_{i_1}-v_1\|$, we have 
\[
\frac{|X_{i_1}'(X_{i_1}-v_1)|}{\|X_{i_1}\|^2} \|v_1\|\leq \frac{\|v_1\|}{\|X_{i_1}\|}\|X_{i_1}-v_1\|\leq \frac{d_{\max}}{d_{\max}-\beta}\|X_{i_1}-v_1\|,
\]
and 
\[
\frac{v_1'X_{i_1}}{\|X_{i_1}\|^2}\leq \frac{\|v_1\|}{\|X_{i_1}\|}\leq \frac{d_{\max}}{d_{\max}-\beta}.
\] Plugging these inequalities into (\ref{tech-lem1-1}) and applying (\ref{thm-SPA-claim1}), we obtain:
\begin{align} \label{thm-SPA-7}
\|\tilde{v}_1\| & \leq \frac{2d_{\max}}{d_{\max}-\beta}\|X_{i_1}-r_{i_1}\|\cr
&\leq \frac{2d_{\max}}{d_{\max}-\beta}\Bigl(\beta+\frac{30\gamma}{\sigma_*}\max\bigl\{1, \frac{d_{\max}}{\sigma_*}\bigr\}\beta\Bigr). 
\end{align}
By our assumption, $\frac{30d_{\max}}{\sigma_*}\max\bigl\{1, \frac{d_{\max}}{\sigma_*}\bigr\}\beta \leq \sigma_*/15$. Moreover, we have shown $d_{\max}\geq \tilde{d}_{\max}\geq \sigma_*/2$. It further implies $\beta\leq \frac{\sigma_*^2}{450d_{\max}}\leq \frac{1}{225}\sigma_* \leq \frac{1}{100}\tilde{d}_{\max}$. As a result, 
\beq \label{thm-SPA-8}
\|\tilde{v}_1\|\leq \frac{200}{99}(\beta +\frac{\sigma_*}{15})\leq  \frac{3}{10}\tilde{d}_{\max}\leq \tilde{d}_{\max}-\frac{7}{20}\sigma_*.
\eeq
At the same time, $h_0=\sigma_*/3$. Hence, 
\[
\|\tilde{v}_1\|< \widetilde{d}_{\max}-h_0\qquad\Longrightarrow\qquad 1\notin \widetilde{\cal K}(h_0). 
\]
This proves the third claim in (\ref{thm-SPA-induction-3}).

\subsubsection{Proof of Lemma~\ref{lem:SPA-tech-lem2}}  \label{subsec:proof-techlemma-2}

Suppose we have already obtained (\ref{thm-SPA-iter-2}) and (\ref{thm-SPA-claim3}) for each $1\leq j\leq s-1$, and we would like to show (\ref{thm-SPA-iter-2}) for $s$. 

First, consider the second claim in (\ref{thm-SPA-iter-2}). For each $k\notin{\cal M}_{s-1}$, it has $(s-1)$ zeros in its barycentric coordinate (corresponding to those indices in ${\cal M}_{s-1}$). We apply Lemma~\ref{lem:simplex} to obtain:
\[
\|\tilde{v}_k-\tilde{v}_{\ell}\|\geq \sqrt{2}\sigma_{K-s}(\widetilde{V})\geq \sqrt{2}\sigma_*, \qquad\mbox{for all $k\neq \ell$ in $\{1,\ldots,K\}\setminus {\cal M}_{s-1}$},
\] 
where the first inequality is from (\ref{lem-simplex2}) and the second inequality is from (\ref{thm-SPA-iter-1}). 

Next, consider the third claim in (\ref{thm-SPA-iter-2}). Note that ${\cal M}_{s-1}=\{k_1, k_2,\ldots,k_{s-1}\}$. For each $1\leq j\leq s-1$, by definition, $\tilde{v}_{k_j} = \bigl[\prod_{m\geq j}(I_d-\hat{y}_m\hat{y}_m')\bigr]\cdot (I_d-\hat{y}_j\hat{y}_j)H_{j-1}v_{k_j}$. 
It follows that
\beq \label{thm-SPA-iter-3}
\|\tilde{v}_{k_j}\|\leq \|(I_d-\hat{y}_j\hat{y}_j)H_{j-1}v_{k_j}\|, \qquad\mbox{where}\quad \hat{y}_j=\frac{H_{j-1}X_{i_j}}{\|H_{j-1}X_{i_j}\|}. 
\eeq
Here, $\|H_{j-1}X_{i_j}\|$ is the maximum Euclidean distance attained in the $(j-1)$th iteration. Since we have already established (\ref{thm-SPA-iter-2}) for $j$, we immediately have 
\[
\|H_{j-1}X_{i_j}\|\geq \sigma_*/2, \qquad \mbox{for }1\leq j\leq s-1. 
\] 
In addition, we have shown (\ref{thm-SPA-claim2}) for $1\leq j\leq s-1$, which implies that 
\[
\|H_{j-1}X_{i_j} - H_{j-1}v_{k_j}\|\leq \Bigl(1+\frac{30\gamma}{\sigma_*}\max\bigl\{1, \frac{d_{\max}}{\sigma_*}\bigr\}\Bigr)\beta. 
\] 
Using the above ineqaulities, we can mimic the proof of (\ref{thm-SPA-7}) to show that 
\beq \label{thm-SPA-iter-4}
\|(I_d-\hat{y}_j\hat{y}_j)H_{j-1}v_{k_j}\|\leq \Bigl(1+\frac{30\gamma}{\sigma_*}\max\bigl\{1, \frac{d_{\max}}{\sigma_*}\bigr\}\Bigr)\beta.
\eeq 
Write $\Gamma_j=I_d-\hat{y}_j\hat{y}_j'$. It is seen that
\[
\|\tilde{v}_{k_j}\| =\Bigl\|\prod_{\ell=j+1}^s \Gamma_{j}H_{j-1}v_{k_j}\Bigr\|\leq \|\Gamma_j H_{j-1}v_{k_j}\| \leq \|(I_d-\hat{y}_j\hat{y}_j)H_{j-1}v_{k_j}\|.
\]
Therefore, for $1\leq j\leq s-1$, 
\beq \label{thm-SPA-iter-5}
\|\tilde{v}_{k_j}\|\leq \Bigl(1+\frac{30\gamma}{\sigma_*}\max\bigl\{1, \frac{d_{\max}}{\sigma_*}\bigr\}\Bigr)\beta.
\eeq
We further mimic the argument in (\ref{thm-SPA-8}) to obtain: 
\[
\|\tilde{v}_{k_j}\|\leq \tilde{\beta}_{\max}-7\sigma_*/20<\tilde{\beta}-h_0, \qquad \mbox{for all }1\leq j\leq s-1. 
\] 
This implies that 
\beq \label{thm-SPA-iter-6}
k_j \notin \widetilde{\cal K}(h_0)\;\; \mbox{for $1\leq j\leq s-1$}\quad\Longrightarrow\quad {\cal M}_{s-1}\cap \widetilde{\cal K}(h_0)=\emptyset. 
\eeq

Last, consider the first claim in (\ref{thm-SPA-iter-2}). Let $\Delta$ denote the right hand side of (\ref{thm-SPA-iter-5}) for brevity. We have shown $\|\tilde{v}_k\|\leq \Delta$, for all $k\in {\cal M}_{s-1}$. By our assumption, we can easily conclude that $\sigma_*^2\geq 2(K-2)\Delta$. We then apply Lemma~\ref{lem:vertexNorm} with $s-1$ and $\delta=\Delta$ to get 
\beq
\tilde{d}_{\max}\geq \frac{1}{2}\sigma_{K-s}(\widetilde{V})\geq \sigma_*/2, 
\eeq  
where the last inequality is from (\ref{thm-SPA-iter-1}).



\section{Proof of the main theorems}
\label{sec:Proof-main} 

We recall our pp-SPA procedure. On the hyperplane, we obtained the projected points
\begin{align*}
\tilde X_i := H ( X_i  - \bar X ) + \bar X = (I_d  - H) \bar X + H r_i + H \epsilon_i 
\end{align*}
after rotation by $U$, they become $Y_i = U'\tilde X_i  =  U'r_i + U' \epsilon_i = U'X_i\in \mathbb{R}^{K-1}$.  Denote $\tilde Y_i = U_0'X_i = U_0'r_i + U_0'\epsilon_i \in \mathbb{R}^{K-1}$. In particular, $U_0'\epsilon_i\sim N(0, \sigma^2I_{K-1})$. Then, without loss of generality, the vertex hunting analysis on $\tilde Y_i$ is equivalent to that of  $ X_i = r_i + \epsilon_i \in \mathbb{R}^{p}, $ where $ \epsilon_i \sim N(0, \sigma^2 I_{p})$ with $p =K-1$. We provide the following theorems for the rate by applying D-SPA on the aforementioned low dimension $p= K-1$ space. The proof of these two theorems are postponed to Section \ref{sub:thmD-SPA}.

\begin{thm} \label{thm:KNN1} 
Consider $ X_i = r_i + \epsilon_i \in \mathbb{R}^{p}, $ where $ \epsilon_i \sim N(0, \sigma^2 I_{p})$ for $1\leq i \leq n$. 
Suppose $m\geq c_1 n$ for a constant $c_1>0$ and $p \ll \log(n)/ \log \log (n)$. Let $ p/\log (n) \ll\delta_n \ll1$.
Let $c_2^* = 0.9 (2e^{2})^{-1/p} \sqrt{(2/p)} (\Gamma (p/2+1))^{1/p}$. Then, $c_2^*\to 0.9e^{-1/2}$ as $p\to \infty$. . 
We apply D-SPA to $X_1, X_2, \ldots, X_n$ and output $X^*_1,\cdots, X^*_n $ where some $X_i^*$ may be NA owing to the pruning. 
If  we choose $N = \log(n) $ and
\[
 \Delta =    c_3 \sigma \sqrt p  \Big( \frac{\log (n)}{n^{1 -\delta_n}}\Big)^{1/p}  
\mbox{for a constant $c_3 \leq c_2^*$}, 
\] 
Then, 
\[
\beta_{new}(X^*) \leq \sqrt{\delta_n}\cdot  \sigma \cdot \sqrt{2\log (n)}
\]
If the last inequality of (\ref{cond:G}  )  and (\ref{cond:SPA2} ) hold, then up to a permutation in the columns, 

\[
\max_{1\leq k \leq K}\|\hat{v}_k - v_k\| \leq  g_{new}(V) \cdot \sqrt{\delta_n}\cdot  \sigma \cdot \sqrt{  2 \log(n)}.  
\] 
\end{thm} 

The second theorem discuss the case there a fewer pure nodes.  
\begin{thm} \label{thm:KNN2} 
Consider $ X_i = r_i + \epsilon_i \in \mathbb{R}^{p}, $ where $ \epsilon_i \sim N(0, \sigma^2 I_{p})$ for $1\leq i \leq n$. 
 Fix $0 < c_0  < 1$ and assume that $m \geq  n^{1-c_0 + \delta}$ for a sufficiently small constant $0<\delta<c_0$. Suppose $p \ll \log(n)/ \log \log (n)$.  
Let $c_2^* = 0.9 (2e^{2-c_0})^{-1/p} \sqrt{(2/p)} (\Gamma (p/2+1))^{1/p}$.   Then $c_2^*\to 0.9e^{-1/2}$ as $p\to \infty$. 
Suppose we apply D-SPA to $X_1, X_2, \ldots, X_n$ and output $X^*_1,\cdots, X^*_n $ where some $X_i^*$ may be NA owing to the pruning. If we choose  $N = \log(n)$ and 
\[
 \Delta= c_3 \sigma \sqrt {p}\Big( \frac{\log (n)}{n^{1-c_0}}\Big)^{1/p} \text { for a constant $c_3\leq c_2^*$}.
\]
Then, 
\[
\beta_{new}(X^*)\leq \sqrt{c_0}\cdot  \sigma \cdot \sqrt{2\log (n)}
\]

If  the last inequality of (\ref{cond:G}  )  and (\ref{cond:SPA2} ) hold, then up to a permutation in the columns, 
\[
\max_{1\leq k \leq K}\|\hat{v}_k - v_k\|\leq  g_{new}(V) \cdot \sqrt{c_0 } \cdot \sigma \sqrt{ 2 \log(n)}. 
\] 
for any arbitrary small constant  $\delta<0$.

  \end{thm}

Based on the above two theorem, we have the results on $\{\tilde Y_i\}'s$. However, what we really care about is on $\{Y_i\}'s$ which differ from $\{\tilde Y_i\}'s$ by the rotation matrix. To bridge the gap, we need the following Lemma.

\begin{lemma}\label{lem:H}
Suppose that $s^2_{K-1}(R)\gg \max\{\sqrt{\sigma^2d/n}, \sigma^2d/n \}$ and $\sigma = O(1)$. Then, with probability $1- o(1)$, 
\begin{align} \label{bdd:H-H0}
 \Vert U - U_0 \Vert \asymp \|H-H_0\|&\leq  \frac{C}{s^2_{K-1}(R)}\max\{\sqrt{\sigma^2d/n}, \sigma^2d/n \}
\end{align}

\end{lemma}

\subsection{Proof of Theorems~\ref{thm:main1} and \ref{thm:main2} }
With the help of Theorems~\ref{thm:KNN1}, \ref{thm:KNN2} and Lemma~\ref{lem:H}, we now prove Theorems~\ref{thm:main1} and \ref{thm:main2}. We will present the detailed proof for  Theorem~\ref{thm:main2}. The proof of Theorem~\ref{thm:main1} is nearly identical to that of Theorem~\ref{thm:main2} with the only difference in employing Theorem~\ref{thm:KNN1}, and we refrain ourselves from repeated details. 

\begin{proof}[Proof of Theorem~\ref{thm:main2}]
Recall that $Y_i = U' X_i = U' r_i + U' \epsilon_i$ and $\tilde Y_i = U_0'r_i + U_0' \epsilon_i$. 
Theorem~\ref{thm:KNN2}  indicates that applying D-SPA on $\bar Y_i$ improves the rate to  $\sigma (1+ o(1))\sqrt{2c_0 \log(n)}$. Note that $\Vert r_i\Vert \leq 1$. Also, by Lemma~\ref{lemma:chi2},  $ \Vert \epsilon_i  \Vert\leq (1+o(1)) \sigma ( \sqrt{\max\{d, 2\log (n)\}})$ simultaneously for all $i$, with high probability. Under the assumption $\alpha_n =o(1)$ for both cases and $s_{K-1}^2 (R) \asymp s^2_{K-1}(\tilde V)$ by Lemma~\ref{lemma:G}, the first condition in  Lemma~\ref{lem:H} is valid. By the last inequality in (\ref{cond:G} ), we have the norm of $r_i$ should be upper bounded for all $1\leq i \leq n$ and therefore $s_{K-1}(\tilde V)\leq C \max_{k\neq l} \Vert \tilde v_k - \tilde v_\ell\Vert \leq C $. Further with the condition (\ref{cond:SPA2} ), we obtain that $\sigma = O(1)$. Therefore, the conditions in Lemma~\ref{lem:H} are both valid. Then by employing Lemma~\ref{lem:H}, we can derive that
\[
\Vert  Y_i - \tilde Y_i  \Vert  = O_{\mathbb P} \left( \frac{\sigma \sqrt d}{\sqrt n s^2_{K-1}(R)}  (1+   \sigma  \sqrt{\max\{d, 2\log (n)\}} \, )\right) =O_{\mathbb P}(\sigma\alpha_n)
\]
where the last step is due to Lemma~\ref{lemma:G} under the condition  (\ref{cond:G} ).

Consider the first case that $\alpha_n\ll t_n^*$. We choose $\Delta= c_3 t_n^*\sigma$. It is seen that $\sigma \alpha_n\ll \Delta$. We will prove by contradiction that applying pp-SPA with $(\Delta, \log (n))$ on $\{Y_i\}$, the denoise step can remove outlying points whose distance to the underlying simplex larger than $\sigma [  \sqrt{2c_0 \log(n)} +  C\alpha_n]$ for some $C>0$.

 First, suppose that with probability $c$ for a small constant $c>0$, there is one point $Y_{i_0}$ away from the underlying simplex by a distance larger than $\sigma [  \sqrt{2c_0 \log(n)} +  C\alpha_n]$ and it is not pruned out. Since $\sigma \alpha_n \ll \Delta$, we see that $\tilde Y_{i_0}$ is faraway to the simplex with distance $\sigma   \sqrt{2c_0 \log(n)} $ for certain large $C$ and it cannot be pruned out by $(1.5\Delta, \log (n))$. Otherwise if it can be pruned out,  $\mathcal B(Y_{i_0}, \Delta) \subset \mathcal B(\tilde Y_{i_0}, 1.5\Delta)$ and hence $N(\mathcal B(Y_{i_0}, \Delta)) \geq \log(n)$, which means that  we can prune out $ Y_{i_0}$ with $(\Delta, \log (n))$. This is a contradiction. However, by employing  Theorem~\ref{thm:KNN2} on $\{\tilde Y_i\}$ with $p =K-1$ and noticing $c_2^* = 1.8 c_2$ with $c_2$ defined in the manuscript, we should be able to prune out $\tilde Y_{i_0}$ with high proability.  This leads to a contradiction.  
 
 Second, suppose that with probability $c$ for a small constant $c>0$, all outliers can be removed but a vertex $v_1$ is also removed (which means all points near it are removed). Then, $N (\mathcal B(v_1, \Delta)) < \log (n)$. For the corresponding vertex for $\{\tilde Y_i\}$, denoted by $\tilde v_1$, it holds that $N (\mathcal B(\tilde v_i, \Delta/2))< \log (n)$ which means the vertex $\tilde v_1$  for $\{\tilde Y_i\}$ is also pruned. However, again by Theorem~\ref{thm:KNN2}, this can only happen with probability $o(1)$. This leads to  another contradiction. 
 
Let us denote by $  \beta(Y^*, U_0' V)$ the maximal distance of points in $Y^*$ to the simplex formed by $ U_0' V$.   By the above  two contradictions, we conclude that with high probability, 
  \[
  \beta(Y^*, U_0' V) \leq \sigma [  \sqrt{2c_0 \log(n)} +  C\alpha_n].
  \]
  where $U_0'V$ is the underlying simplex of $\{\tilde Y_i\}$. It is worth noting that $\alpha_n = o(1)$. Then, under the assumptions of the theorem, we can apply Theorem~\ref{thm:SPA} (Theorem 1 in the manuscript). It gives that 
  \begin{align*}
  \max_{1\leq k \leq K}\Vert \hat v^*_k - U_0'v_k\Vert  \leq \sigma g_{new}(V)[  \sqrt{2c_0 \log(n)} +  C\alpha_n]
  \end{align*}
  where we use $ (\hat v_1^*, \cdots, \hat v_K^*)$ to denote the output vertices by applying SP on $\{Y_i\}$. Eventually, we output each vertex $\hat v_k = (I_K - UU') \bar X + U \hat v^*_k$. It follows that up to a permutation of the $K$ vectors,  
  \begin{align*}
   \max_{1\leq k \leq K}\Vert \hat v_k - v_k\Vert & \leq    \max_{1\leq k \leq K}\Vert U \hat v^*_k  - v_k \Vert + \Vert (I_d - UU')\bar X - (I_d - U_0U_0')\bar r\Vert\notag\\
   &  \leq    \max_{1\leq k \leq K}\Vert  \hat v^*_k  -U_0'v_k \Vert +\Vert U - U_0\Vert + \Vert (I_d - UU')\bar X - (I_d - U_0U_0')\bar r\Vert
  \end{align*}
  Further we can derive 
  \begin{align*}
  \Vert (I_d - UU')\bar X - (I_d - U_0U_0')\bar r\Vert &  \leq \Vert H- H_0 \Vert + \Vert \bar X - \bar r\Vert  \notag\\
  &\leq \sigma \alpha_n + \Vert \bar\epsilon \Vert  \notag\\
  & \leq \sigma \alpha_n  + \frac{2\sigma \sqrt{\max\{d, 2\log (n)\}} }{\sqrt n}
  \end{align*}
  this together with Lemma~\ref{lem:H}, give rise to 
  \begin{align*}
   \max_{1\leq k \leq K}\Vert \hat v_k - v_k\Vert & \leq   \sigma g_{new}(V)[  \sqrt{2c_0 \log(n)} +  C\alpha_n] + \frac{2\sigma \sqrt{\max\{d, 2\log (n)\}} }{\sqrt n} \, .
  \end{align*}
  
Consider the second case that $\alpha_n\gg t_n^*$ where we choose $\Delta = \sigma \alpha_n$. By Lemma~\ref{lemma:chi2}, it is observed that with high probability, $\max_{1\leq i\leq n}d(\tilde Y_i, \mathcal S ) < (1+ o(1))\sigma \sqrt {2\log (n)} $. Notice that $\Vert  Y_i - \tilde Y_i  \Vert  \leq C \sigma \alpha_n$ with high probability. For $Y_i$, if its distance to the underlying simplex is larger than $\sigma[ (1+ o(1)) \sqrt {2\log (n)} + C_1\alpha_n]   $ for a sufficiently large $C_1>3C+ 1$, then  $d(\tilde Y_i, \mathcal S) \geq d(Y_i, \mathcal S) - C\sigma\alpha_n > \sigma[ (1+ o(1)) \sqrt {2\log (n)} + (2C+1)\alpha_n]   $. Hence, $\mathbb B(\tilde Y_i, (2C+1)\Delta ) )$ is away from the simplex by a distance larger than $\sigma(1+ o(1)) \sqrt {2\log (n)}$. It follows that $N(\mathbb B(Y_i, \Delta)) \leq N(\mathbb B(\tilde Y_i, (2C+1) \Delta ) )< \log (n)$. This is equivalent to say that we prune out the points there. Consequently, 
with high probability, 
  \[
  \beta(Y^*, U_0' V) \leq \sigma[ (1+ o_{\mathbb P}(1)) \sqrt {2\log (n)} + C_1\alpha_n]
  \]
  and further by Theorem~\ref{thm:SPA} (Theorem 1 in the manuscript),  
  \begin{align*}
  \max_{1\leq k \leq K}\Vert \hat v^*_k - U_0'v_k\Vert  \leq \sigma g_{new}(V)[  \sqrt{2 \log(n)} +  C_1\alpha_n]
  \end{align*}
  Next, replicate the proof for $ \max_{1\leq k \leq K}\Vert \hat v_k - v_k\Vert$ in the former case, we can conclude that 
    \begin{align*}
   \max_{1\leq k \leq K}\Vert \hat v_k - v_k\Vert & \leq    \sigma g_{new}(V)[ (1+ o_{\mathbb P}(1)) \sqrt {2\log (n)} + C_1\alpha_n]+ \frac{2\sigma \sqrt{\max\{d, 2\log (n)\}} }{\sqrt n} 
   \notag\\
   &=   \sigma g_{new}(V)(1+ o_{\mathbb P}(1)) \sqrt {2\log (n)}.
  \end{align*}
This concludes our proof.

\end{proof}

\subsection{Proof of Theorems~\ref{thm:KNN1} and \ref{thm:KNN2}. }\label{sub:thmD-SPA}

In the subsection, we provide the proofs of Theorems~\ref{thm:KNN1} and \ref{thm:KNN2}. We show the proof of Theorem~\ref{thm:KNN2} in detail and briefly present the proof of Theorems~\ref{thm:KNN1} as it is similar to that of Theorem~\ref{thm:KNN2}. 

\begin{proof} [Proof of Theorem~\ref{thm:KNN2}]
We first claim the limit of  $c_2^* = 0.9 (2e^{2-c_0})^{-1/p} \sqrt{(2/p)} (\Gamma (p/2+1))^{1/p}$. 
Note that $\Gamma(p/2+1) = (p/2)! $ if $p$ is even and $\Gamma(p/2+1) = \sqrt \pi (p+1)! /( 2^{p+1}(\frac{p+1}{2})!)  $ if $p$ is odd. Using Stirling's approximation, it is elementary to deduce that 
\[
 c_2^*  =e^{O(1/p) - (1- \log (p+1))(p+1) /2p - \log(p) /2} \to e^{-1/2}. 
 \] 
 
Define  the radius $\Delta \equiv \Delta_n= c_3 \sigma \sqrt {p}\Big( \frac{\log (n)}{n^{1-c_0}}\Big)^{1/p} $ for a constant $c_3\leq c_2$. In the sequel, we will prove that applying D-SPA to $X_1, \cdots, X_n$ with $(\Delta, N)$, we can prune out the points whose distance to the underlying true simplex are larger than the rate in the theorem, while the points around vertices are captured. 

 Denote $d(x, \mathcal{S}) $, the distance of $x$ to the simplex $\mathcal S$. Let 
\begin{align*}
\mathcal{R}_f: = \{x\in \mathbb{R}^{p}: d(x, \mathcal{S}) \geq  2\sigma \sqrt{\log (n)}\,  \}
\end{align*}
We first claim that the number of points in $\mathcal{R}_f$, denoted by $N(\mathcal{R}_f)$, is bounded with probability $1- o(1)$. By definition, we deduce 
\begin{align*}
N(\mathcal{R}_f) =  \sum_{i=1}^n \mathbf{1}(x_i \in\mathcal{R}_f) \leq \sum_{i=1}^n \mathbf{1} ( \Vert  \varepsilon_i \Vert \geq 2 \sigma \sqrt{\log n} \, )   
\end{align*}
The mean on  the RHS is given by $n \mathbb{P} (\Vert  \varepsilon_i \Vert \geq  2 \sigma \sqrt{\log n}) = n \mathbb{P} (\chi^2_{p}\geq  4\log n) \leq n e^{- 1.5\log (n)}= n^{-1/2}$. By similar computations, the order of the variance is again $ n^{-1/2}$. By Chebyshev's inequality, we conclude that $N(\mathcal{R}_f) = o_{\mathbb P}(1) $. 

In the sequel, we use the notation $\mathbb{B} (x, r)$ to represent a ball centered at $x$ with radius $r$ and denote $N(\mathbb{B} (x, r))$ the number of points falling into this ball. And we also 
denote $\mathcal S$ the true underlying simplex. 

Based on these notation, we introduce 
\begin{align*}
P:= & \mathbb P ( \exists \,\,   X_i \text{ satisfying } \sigma\sqrt{2c_0 \log (n)} \leq d(X_i, \mathcal S) \leq 2\sigma \sqrt{\log (n)} \text { cannot be pruned out } )
\end{align*}
We aim to show that $P = o(1)$. To see this, we first derive 
\begin{align*}
P & =    {n \choose N}  N \cdot \mathbb P (X_1, \cdots X_N \in \mathcal \mathcal B(X_1, \Delta) \text{ s.t. $ \sigma\sqrt{2c_0 \log (n)}  \leq d(X_1, \mathcal S) \leq 2\sigma \sqrt{\log (n)} $} \,  ) \notag\\
&\leq {n \choose N} N \cdot \int_{a_n  \leq d(x, \mathcal S) \leq b_n} f_{X_1}(x)
\mathbb P (X_2, \cdots, X_N \in \mathcal B(x, \Delta)  ) {\rm d} x\notag\\
& \leq {n \choose N} N \cdot \int_{a_n  \leq d(x, \mathcal S) \leq b_n} f_{X_1}(x)
\prod_{t=2}^N \mathbb P (X_t\in \mathcal B(x, \Delta) )  {\rm d} x
\end{align*}
where $a_n : = \sigma\sqrt{2c_0 \log (n)} $ and $b_n:= 2\sigma \sqrt{\log (n)}$ for simplicity. 
We can  compute that for any $2\leq t \leq N$,
\begin{align}\label{eq:2024031601}
\mathbb P (X_t\in \mathcal B(x, \Delta) )  &= (2\pi\sigma^2 )^{-\frac{p}{2}}\int_{\Vert y - x\Vert \leq \Delta}  \exp\{ - {\Vert y - r_t\Vert^2}/{2\sigma^2}\} {\rm d} y  \notag\\
&\leq  \frac{ (\Delta/\sigma)^p}{ 2^{p/2}\Gamma(p/2+1)}\exp \Big\{ - \frac{(\Vert x - r_t\Vert - \Delta)^2}{2\sigma^2}\Big\} \notag\\
& \leq  (\Delta/\sigma)^p C_p \exp \Big\{ - \frac{\Vert x- r_t\Vert ^2}{2(1+ \tau_n)\sigma^2}\Big\}
\end{align}
where $\tau_n := C \Delta/ \sigma\sqrt{2c_0\log(n)}$ for a large $C>0$;and we write $C_p: = 2^{1-p/2}/\Gamma(p/2+1)$. Here to obtain the last inequality, we used the definition of $\Delta$ and the derivation
\[
\frac{\Delta}{\| x- r_t\|} \leq \frac{ \Delta} {\sigma \sqrt{2c_0 \log(n)}} \leq C\tau_n \leq C \sqrt p (\log(n))^{1/p-1/2} /n^{(1-c_0)/p}= o(1) 
\]
so that 
\[
(1- {\Delta}/{\| x- r_t\|})^2\leq (1+ \tau_n)^{-1} 
\]
by choosing appropriate $C$ in the definition of $\tau_n$. Further, under the condition that $p \ll \log(n) /\log \log(n)$, one can verify that 
\[
\tau_n\ll 1/\log(n)= o(1) \,. 
\]
(\ref{eq:2024031601}), together with 
\[
f_{X_1} (x) = (2\pi\sigma^2)^{-\frac{p}{2}} \exp\{ -\Vert x - r_1\Vert^2 /(2\sigma^2)\} \leq  (2\pi\sigma^2)^{-\frac{p}{2}} \exp\{ -\Vert x - r_1\Vert^2 /(2(1+ \tau_n)\sigma^2)\} ,
\]
leads to 
\begin{align*}
P \leq    {n \choose N} N C_p^{N-1} (\Delta/\sigma)^{p(k-1)}\cdot  \int_{a_n  \leq d(x, \mathcal S) \leq b_n}(2\pi\sigma^2)^{-\frac{p}{2}} \exp  \Big\{ - \frac{\sum_{t=1}^N\Vert x- r_t\Vert ^2}{2(1+ \tau_n)\sigma^2}\Big\}{\rm d} x
\end{align*}
Also, notice that $\sum_{t=1}^N\Vert x- r_t\Vert ^2 \geq N \Vert x- \bar r\Vert^2$ where $\bar r = N^{-1} \sum_{t=1}^N r_t$. Then, 
\begin{align*}
P &\leq  {n \choose N} NC_p^{N-1}   (\Delta/\sigma)^{p(N-1)} \cdot\int_{a_n  \leq d(x, \mathcal S) \leq b_n} (2\pi\sigma^2)^{-\frac{p}{2}}\exp \Big\{ - \frac{N \Vert x- \bar r\Vert^2}{2(1+ \tau_n)\sigma^2}\Big\}{\rm d} x\notag\\
& \leq {n \choose N} N C_p^{N-1}   (\Delta/\sigma)^{p(N-1)}  \int_{\Vert x- \bar r\Vert\geq a_n} (2\pi\sigma^2)^{-\frac{p}{2}} \exp \Big\{ - \frac{N \Vert x- \bar r\Vert^2}{2(1+ \tau_n)\sigma^2}\Big\}{\rm d} x\notag\\
& \leq  {n \choose N} NC_p^{N-1}   (\Delta/\sigma)^{p(N-1)} N^{-p/2} (1+ \tau_n)^{p/2} \cdot\mathbb P (\chi^2_{p}\geq 2N c_0\log n/ (1+ \tau_n))
\end{align*}
where we used the fact that $\Vert x- \bar r\Vert \geq d(x, \mathcal{S})$ in the second step and we did change of variables so that the integral reduces to the tail probability of $\chi^2_{p}$ distribution. By Mills ratio,  the tail probability of $\chi^2_{p}$ is given by 
\begin{align*}
\mathbb P (\chi^2_{p}\geq 2Nc_0\log n/(1+ \tau_n))\leq  Cn^{-Nc_0/(1+ \tau_n)} \big(2Nc_0\log n/(1+ \tau_n) \big)^{p/2-1}, 
\end{align*}
we obtain
\begin{align*}
P \leq C {n \choose N} NC_p^{N-1}   (\Delta/\sigma)^{p(N-1)} N^{-p/2}  n^{-Nc_0/(1+ \tau_n)} (2Nc_0\log n)^{p/2-1}\, . 
\end{align*}
Using the approximation ${n \choose k} \leq C (en/k)^k$, we deduce that 
\begin{align*}
P &\leq C \left[ e (2Nc_0\log n)^{(p-2)/(2N)}C_p^{1-1/N} N^{(1-p/2)/N}  \cdot \frac{n^{1-c_0/(1+ \tau_n)} (\Delta/\sigma)^{p(1-1/N)}}{N}\right]^{N}  \notag\\
&= :C \Big[A(n,p, N) \cdot  \frac{n^{1-c_0/(1+ \tau_n)} (\Delta/\sigma)^{p(1-1/N)}}{N} \Big]^N 
\end{align*}
Now we plug in $N =  \log (n)$ and  $\Delta= c_3 \sigma \sqrt {p}\Big( \frac{\log (n)}{n^{1-c_0}}\Big)^{1/p}$ for a constant $c_3\leq c_2$ where $c_2= 0.9 (2e^{2-c_0})^{-1/p} \sqrt{(2/p)} (\Gamma (p/2+1))^{1/p} = 0.9 e^{-(2-c_0)/p} C_p^{-1/p} /\sqrt p $ with $C_p =  2^{1-p/2}/\Gamma(p/2+1)$. It is straightforward to compute that 
\begin{align*}
&\quad A(n,p, N) \cdot  \frac{n^{1-c_0/(1+\tau_n)} (\Delta/\sigma)^{p(1-1/N)}}{N} \notag\\
& \leq e^{1-(2-c_0)(1-1/\log (n))} 2^{\frac{p-2}{2\log (n)}} (c_0\log (n))^{\frac{p-2}{2\log (n)}} (0.9)^{p(1-1/\log (n))}n^{\tau_n c_0/(1+ \tau_n) } \Big( \frac{n^{1-c_0}}{\log (n)}\Big)^{1/\log (n)} \notag\\
& \leq  e^{o(1)}(0.9)^{p} < 1.01\cdot 0.9 <1
\end{align*}
under the condition that $p \ll \log(n) /\log \log (n)$, which also give rise to $\tau_n \log(n) = o(1)$. This implies $P \leq C (0.909)^{\log (n)} = o(1)$.

%
%
%

In the mean time, for each vertex $v_k$,  recall that $J_k =\{ i: r_i =v_k\}$, 
\begin{align*}
N( \mathcal{B}(v_k, \Delta/2)) \geq \sum_{i\in J_k} \mathbf{1}(x_i \in \mathcal{B}(v_k, \Delta/2)) =  \sum_{i\in J_k} \mathbf{1} ( \Vert \varepsilon_i \Vert \leq  \Delta/2) \geq mp_\Delta - C\sqrt{ m p_\Delta \log\log(n)}.
\end{align*}  
with probability $1- o(1)$, and  
\begin{align*}
p_\Delta: = \mathbb{P} (\Vert \varepsilon_i \Vert \leq  \Delta/2) =\mathbb{P} (\chi_{p}^2 \leq 4^{-1}(\Delta/\sigma)^2)\geq \frac{e^{-(\Delta/\sigma)^2/8} 2^{-p}}{2^{p/2}\Gamma (p/2+1)}(\Delta/\sigma)^p
\end{align*}
Recall the condition that $m\geq  n^{\delta} n^{1-c_0}$. It follows that 
\begin{align*}
mp_\Delta \geq  n^{\delta} \frac{e^{-(\Delta/\sigma)^2/8} 2^{-p}}{2^{p/2}\Gamma(p/2+1)}  n^{1-c_0}(\Delta/\sigma)^p  &=  n^{\delta}\frac{e^{-(\Delta/\sigma)^2/8}}{2^{p/2}\Gamma(p/2+1)}   \cdot \frac{c\log (n)}{C_p}2^{-p} (c_3/c_2)^p  \notag\\
& \geq  c  n^{\delta} 2^{-p}(c_3/c_2)^p \log (n)\gg \log (n)
\end{align*}
where  $c>0$ is some small constant. The last step is due to the fact that $n^{\delta}2^{-p} (c_3/c_2)^p = e^{\delta \log (n)- p \log (2c_2/c_3)} \gg 1$ as $2c_2/c_3\geq 2$ is a constant and $p \ll \log (n) /\log \log (n)$. 
Thus, with probability $1- o(1)$, $N( \mathcal{B}(v_k, \Delta/2)) \gg \log (n)$. Under this event, for any point $X_{i_0} \in \mathcal{B}(v_k, \Delta/2)$,  immediately $\mathcal{B}(v_k, \Delta/2) \subset \mathcal{B}(X_{i_0}, \Delta)$  and further $N( \mathcal{B}(X_{i_0}, \Delta)) \gg \log (n)$. 
Combining this, with $P= o(1)$ and $N(\mathcal{R}_f) = o_{\mathbb P}(1) $, we conclude that we can prune out all points with a distance to the simplex larger than $\sigma \sqrt{2c_0\log (n)}$ while preserve those points near vertices, with high probability. Thus we finish the claim for $\beta_{new}(X^*)$. 

The last claim follows directly from Theorem~\ref{thm:SPA} (Theorem 1 in the manuscript) under condition (\ref{cond:SPA2}). We therefore conclude the proof.  


\end{proof}

We briefly present the proof of Theorem~\ref{thm:KNN1} below. 
\begin{proof}
The proof strategy is roughly the same as that of Theorem~\ref{thm:KNN2}
When $m >c_1 n$, we take $ \Delta =    c_3 \sigma \sqrt p  \Big( \frac{\log (n)}{n^{1 -\delta_n}}\Big)^{1/p} $ where $p/\log (n) \ll\delta_n \ll1$ and $c_3\leq c_2$, then similarly we can derive that $N( \mathcal{B}(v_k, \Delta/2))\geq c \log (n) n^{\delta_n} a^{p} = c \log (n)  e^{\delta_n \log (n) - p\log (1/a)}\gg \log (n)$ where $c>0$ is a small constant and $0<a\leq 1$. This gives rise to the conclusion that with high probability, $N( \mathcal{B}(X_{i_0}, \Delta)) \gg \log (n)$ for any $X_{i_0}\in N( \mathcal{B}(v_k, \Delta/2))$.Moreover, in the same manner to the above derivations, replacing $c_0$ by $\delta_n$, we can claim again that $N(\mathcal{R}_f) = o_{\mathbb P}(1) $ and 
\begin{align*}
P &  \leq C\left( A(n,p, \log (n)) \cdot  \frac{n^{1-\delta_n/(1+\tau_n)} (\Delta/\sigma)^{p(1-1/\log (n))}}{\log (n)} \right)^{\log (n)}= o(1). 
\end{align*}
Consequently, all the claims follow from the same reasoning as the proof of Theorem~\ref{thm:KNN2}. We therefore omit the details and conclude the  proof . 
\end{proof}

\subsection{Proof of Lemma~\ref{lem:H}}

Recall that $R=n^{-1/2}[r_1-\bar r,\ldots,r_n-\bar r]$. Let $R=U_0D_0V_0$ be its singular value decomposition and let $H_0=U_0U_0'$. Denote $\epsilon=[\epsilon_1,\ldots,\epsilon_n]\in\mathbb R^{d,n}$.
We start by analyzing the convergence rate of $\|ZZ'-nR  R '-n\sigma^2I_d\|$. Recall that $\bar X=\bar r+\bar\epsilon$, where $\bar\epsilon=n^{-1}\sum_{i=1}^n\epsilon_i$. We obtain
\begin{align}\label{decompo}
Z=X_i-\bar X=r_i+\epsilon_i-\bar r-\bar\epsilon,\qquad Z=\sqrt nR +\epsilon-\bar\epsilon 1_n'.
\end{align}
Observing the fact that $R 1_n=0$, we deduce
\begin{align}\label{decomp}
&ZZ'-nR  R '-n\sigma^2I_d=(\sqrt {n}R +\epsilon-\bar\epsilon 1_n')(\sqrt {n}R +\epsilon-\bar\epsilon 1_n')'-nR  R '-n\sigma^2I_d\notag\\
&\qquad=\sqrt {n}(\epsilon-\bar\epsilon 1_n') R '+\sqrt {n}R (\epsilon-1_n\bar\epsilon')'+(\epsilon-\bar\epsilon 1_n')(\epsilon-\bar\epsilon 1_n')'-n\sigma^2I_d\notag\\
&\qquad=\sqrt {n}\epsilon R '+\sqrt {n}R \epsilon'+(\epsilon\epsilon'-n\sigma^2I_d)-n\bar\epsilon\bar\epsilon'.
\end{align}
The above equation implies that
\begin{align}\label{lemma1:2}
\|ZZ'-nR  R '-n\sigma^2I_d\|\leq 2\sqrt {n}\|\epsilon R'\|+\|\epsilon\epsilon'-n\sigma^2I_d\|+n\|\bar\epsilon\|^2.
\end{align}

We proceed to bound the three terms $\|\epsilon R'\|$, $\|\epsilon\epsilon'-n\sigma^2I_d\|$ and $n\|\bar\epsilon\|^2$ respectively. 
First, notice that $\epsilon R' \in \mathbb{R}^{d\times d}$ is a Gaussian random matrix with independent rows which follow $N(0, RR')$. By Theorem 5.39 and Remark 5.40 in \cite{vershynin}, we can deduce that with probability $1- o(1)$, 
\begin{align*}
n\Vert  R\epsilon'\epsilon R ' \Vert \leq Cn d\sigma^2s_1^2(R).
\end{align*}
This, together with the fact that $s_1(R) \leq c $ gives that  
\begin{align}\label{lemma1:3}
\sqrt n\Vert \mathcal \epsilon R'+ R\epsilon'\Vert \leq C\sigma\sqrt{nd} .
\end{align}
Second, by Bai-Yin law (\cite{bai2008limit}),  we can estimate the bound of $\Vert \mathcal E\mathcal E'- n \sigma^2I_d\Vert $ as follows.  
\begin{align}\label{J2}
\Vert \epsilon\epsilon'- n \sigma^2I_d\Vert \leq n \sigma^2(2\sqrt{d/n} +d/n) \leq \sigma^2 (2\sqrt{nd} + d),
\end{align}
with probability $1- o(1)$. Third, observe that $\bar \epsilon \sim N(0, \sigma^2/n I_d)$. We therefore obtain that with probability $1- o(1)$, $$n\|\bar\epsilon\|^2\leq \sigma^2 [d + C \sqrt{d \log  (n)}] .$$ 
By applying the condition that $\sigma=O(1)$, combining the above equation with (\ref{lemma1:2}), (\ref{lemma1:3}) and (\ref{J2}) yields that, with probability at least $1-o(1)$,
\begin{align}\label{lemma1:4}
\|ZZ'-nR  R '-n\sigma^2I_d\|&\leq 2 \sigma\sqrt{nd}  + \sigma^2 [d + C \sqrt{d \log  (n)}] + \sigma^2 (2\sqrt{nd} + d) \notag\\
& \leq C (\sigma\sqrt{nd}  + \sigma^2 d) .
\end{align}

Now, we compute the bound for $\|\widehat H-H_0\|$. Let $U^\perp,U_0^\perp\in\mathbb R^{d,d-K+1}$ such that their columns are the last $(d-K+1)$ columns of $U$ and $U_0$, respectively. It follows from direct calculations that
\begin{align*}
&\|\widehat H-H_0\|=\|U_0U_0'- U U'\|\leq\|U_0^\perp(U_0^\perp)'(U_0 U_0'-UU')\|+\|U_0U_0'(U_0 U_0'-UU')\|\\
&
=\|U_0^\perp(U_0^\perp)'UU'\|+\|U_0U_0'U^\perp (U^\perp)'\|\leq\|(U_0^\perp)'U\|+\|U_0'U^\perp\|=2\|\sin\Theta(U_0,U)\|.
\end{align*}
Notably, $U, U^{\perp}$ is also the eigen-space of $ZZ'  - n\sigma^2 I_d$. By Weyl's inequality (see, for example, \cite{HornJohnson}), 
\begin{align*}
\max_{1\leq i \leq d} \big| \lambda_i (ZZ'  - n\sigma^2 I_d) - \lambda_i (nRR') \big| \leq C \Vert ZZ'  - n\sigma^2 I_d - nRR' \Vert
\end{align*}
Under the condition that $s^2_{K-1}(R)\gg \max\{\sqrt{\sigma^2d/n}, \sigma^2d/n \}$, by Davis-Kahan Theorem (\cite{sin-theta}), we deduce that, with probability at least $1-o(1)$,
\begin{align}\label{hh0}
\|\widehat H-H_0\|&\leq2\|\sin\Theta(U_0,U)\|\leq\frac{2\|ZZ'-nR  R '-n\sigma^2I_d\|}{\lambda_{K-1} (nRR')}\notag\\
&\leq C\frac{\max\{\sqrt{\sigma^2d/n}, \sigma^2d/n \}}{s^2_{K-1}(R)}.
\end{align}
The proof is complete.

\section{Numerical simulation for Theorem \ref{thm:SPA-maintext}} \label{supp:Example}

In this short section, we want to provide a better sense of our bound derived in Theorem \ref{thm:SPA-maintext} and how it compares with the one from the orthodox SPA. To make it easier for the reader to see the difference between the two bounds, we consider toy example where we fix $(K, d) = (3,3)$ and
\[
\widetilde{V} = \{(20,20,0),(20,30,0),(30,20,0)\}
\] 
while we let 
\[
V = \widetilde{V} + a \cdot (0,0,1). 
\] 

We consider $50$ different values for $a$ ranging from $10$ to $1000$. It is not surprising to see that when $a$ is close to $0$ the bound of the orthodox SPA goes to infinity whereas as the simplex is bounded far away from the origin, the $K^{th}$ singular value will be bounded away from $0$. However, our bound still outperforms the traditional SPA bound even for very large values of $a$. Looking at two specific values of $a$ we have the following. For $a=10$,

\[
\beta_{new} = 0.03 , \qquad \beta(V) = 0.05
\] 
Moreover, as $a$ changes, the Figure \ref{fig:factor} below illustrate how much the ratio of 
\[ 
\frac{\mbox{our whole bound}}{\mbox{Gillis bound}} 
\] 
changes as the parameter $a$ changes. For example, when $a = 10$.
\[
\frac{g_{new}(V)}{g(V)} = 0.015,  
\] 
and so 
\[ 
\frac{\mbox{our whole bound}}{\mbox{Gillis bound}} = 0.009 
\] 
so we reduce the bound by $111$ . Similarly, when $a = 1000$, 
\[
\frac{g_{new}(V)}{g(V)} = 0.19 , \qquad \frac{\mbox{our whole bound}}{\mbox{Gillis bound}} = 0.105, 
\] 
so we have reduced the bound by $9.5$. 

\begin{figure}[ht]
    \centering
    \includegraphics[scale=0.5]{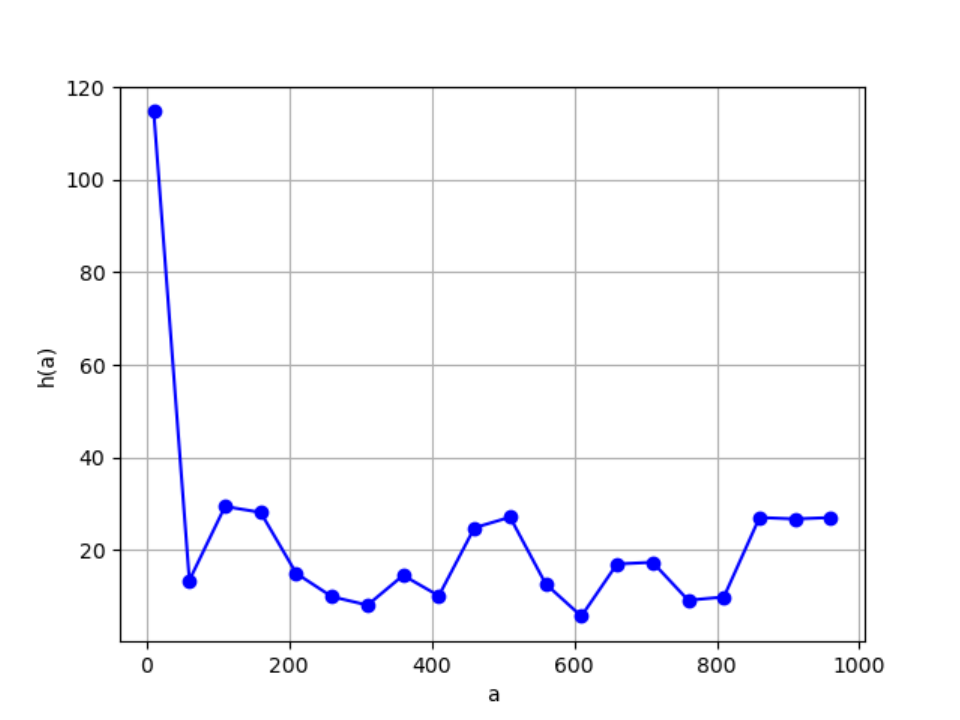}
    \caption{Factor of improvement of our bound over orthodox spa as the true simplex moves away from origin by a distance $a$.}
    \label{fig:factor}
\end{figure}

\bibliography{iclr2024_conference}
\bibliographystyle{iclr2024_conference}

\end{document}

%% file: math_commands.tex
\usepackage{amsmath,amsfonts,bm}

\def\eqref#1{equation~\ref{#1}}

\def\1{\bm{1}}

\def\eps{{\epsilon}}

\DeclareMathAlphabet{\mathsfit}{\encodingdefault}{\sfdefault}{m}{sl}
\SetMathAlphabet{\mathsfit}{bold}{\encodingdefault}{\sfdefault}{bx}{n}

\DeclareMathOperator*{\argmax}{arg\,max}

\newtheorem{theorem}{Theorem}
\newtheorem{lemma}{Lemma}
\newtheorem{definition}{Definition}[section]